\documentclass{article}

\usepackage{microtype}
\usepackage{graphicx}
\usepackage{subfigure}
\usepackage{booktabs} 

\usepackage{hyperref}

\usepackage[accepted]{icml2020}

\icmltitlerunning{Kalman meets Bellman: Improving Policy Evaluation through Value Tracking} 

\usepackage{times}

\usepackage{pifont}

\usepackage{natbib}

\usepackage{algorithm}
\usepackage{algorithmic}
\usepackage{multirow}
\usepackage{cancel}
\usepackage{empheq}

\usepackage{bm}
\usepackage{amssymb,amsmath,amsfonts,graphicx,amsthm}
\usepackage{chngcntr}
\usepackage{apptools}
\newtheorem{theorem}{Theorem}
\newtheorem{assumption}{Assumption}
\newtheorem{corollary}{Corollary}
\newtheorem{lemma}{Lemma}
\AtAppendix{\counterwithin{theorem}{section}}
\AtAppendix{\counterwithin{assumption}{section}}
\AtAppendix{\counterwithin{corollary}{section}}
\AtAppendix{\counterwithin{lemma}{section}}
\AtAppendix{\counterwithin{figure}{section}}

\usepackage{helvet} 
\usepackage{courier}  

\usepackage{algorithm}
\usepackage{algorithmic}
\usepackage{multirow}
\usepackage{cancel}
\usepackage{empheq}
\usepackage{bm}
\usepackage{amssymb,amsmath,amsfonts,amsthm}
\usepackage{chngcntr}
\usepackage{apptools}
\usepackage{pifont}
\usepackage{enumitem}

\usepackage{subfigure}
\usepackage{booktabs} 
\usepackage{authblk}
\usepackage{fancyhdr}
\begin{document} 
	
	\title{\Large \bf Kalman meets Bellman: Improving Policy Evaluation through Value Tracking}
	
	\author[]{Shirli Di-Castro Shashua \thanks{Technion, Israel. shirlidi@tx.technion.ac.il}}
	\author[]{Shie Mannor \thanks{Technion, Israel. shie@ee.technion.ac.il}}
	
	\affil[]{}
	\date{}

	
	\twocolumn
	
	
\maketitle
\thispagestyle{empty}
\begin{abstract}
	Policy evaluation is a key process in Reinforcement Learning (RL). It assesses a given policy by estimating the corresponding value function. When using parameterized value functions, common approaches minimize the sum of squared Bellman temporal-difference errors and receive a point-estimate for the parameters. Kalman-based and Gaussian-processes based frameworks were suggested to evaluate the policy by treating the value as a random variable. These frameworks can learn uncertainties over the value parameters and exploit them for policy exploration. When adopting these frameworks to solve deep RL tasks, several limitations are revealed: excessive computations in each optimization step, difficulty with handling batches of samples which slows training and the effect of memory in stochastic environments which prevents off-policy learning. In this work, we discuss these limitations and propose to overcome them by an alternative general framework, based on the extended Kalman filter. We devise an optimization method, called Kalman Optimization for Value Approximation (KOVA) that can be incorporated as a policy evaluation component in policy optimization algorithms. KOVA minimizes a regularized objective function that concerns both parameter and noisy return uncertainties.  We analyze the properties of KOVA and present its performance on deep RL control tasks. 
\end{abstract}

\section{Introduction}
\label{introduction}
Reinforcement Learning (RL) solves sequential decision-making problems by designing an agent that interacts with the environment and seeks an optimal policy \cite{sutton1998reinforcement}. During the learning process, the agent is required to evaluate its policies using a value function. In many real world RL domains, such as robotics, games and autonomous driving, state and action spaces are large; hence, the value is approximated by a parameterized function. Common approaches for estimating the value function parameters minimize the sum of squared Bellman Temporal Difference (TD) errors \cite{dann2014policy}. These approaches treat the value as a {\em deterministic function} and calculate a point-estimate of the parameters. Alternatively, treating the value or its parameters as {\em random variables} (RVs) has been proposed by \citet{engel2003bayes,engel2005reinforcement}. In their algorithm, called GPTD, they used Gaussian-Processes (GPs) for both the value and the return to capture their uncertainty in policy evaluation. However, GPTD does not consider the possible non-stationarity of the value, that can arise from changes in the policy, changes in the MDP and from bootstrapping \cite{phua2007tracking}. It encourages a {\em tracking} property rather than a convergence property, which has been shown to be desirable even in stationary environments \cite{sutton2007role}. 

To alleviate the effect of non-stationarity, \citet{geist2010kalman} proposed an algorithm, called KTD, which uses the Unscented Kalman filter (UKF) \cite{wan2000unscented}, a non-linear Kalman filter \cite{kalman1960new}, to learn the uncertainty in value parameters. Kalman filters are used for online tracking and for estimating states in dynamic environments through indirect noisy observations. These methods have been successfully applied to numerous control dynamic systems such as navigation and tracking targets \cite{sarkka2013bayesian}. KTD tracks the value parameters by modeling them as RVs following a random walk. 

Policy evaluation for {\em deep RL tasks} consist of several {\em challenges}: learning non-linear approximations such as Deep Neural Networks (DNNs); optimizing over batches of transitions, possibly independent; flexibility in evaluating the policy in both on-policy and off-policy settings and flexibility in using possibly multi-step transitions. Both GPTD and KTD enjoy desirable features of learning uncertainties and exploit them for policy exploration. However, when considering deep RL tasks, we will show that these algorithms suffer from some major limitations: (i) in stochastic environments, they are sensitive to a memory effect that prevent them from being used in off-policy settings. (ii) KTD cannot handle batches of samples, and thus suffers from slow training process. GPTD can handle only batches of successive transitions samples, and thus cannot randomly sample from an experience buffer. (iii) The UKF used in KTD requires excessive computations in each optimization step, which scales poorly with the parameters dimension. 

In this work we preserve the Kalman properties of learning uncertainties and tracking the value parameters while boosting these Kalman-based approaches to handle deep RL tasks. For this purpose we devise KOVA ({\bf K}alman {\bf O}ptimization for {\bf V}alue {\bf A}pproximation): a general framework for addressing uncertainties while approximating value-based functions in deep RL  domains. KOVA combines the Extended Kalman filter (EKF) \cite{anderson1979optimal,gelb1974applied} estimation techniques with RL principles to improve value approximation. KOVA can be incorporated as a policy evaluation component in popular policy optimization algorithms. Unlike GPTD and KTD, KOVA is feasible when using non-linear approximation functions as DNNs and can estimate the value in both on-policy and off-policy settings. To the best of our knowledge, this is the first attempt of adapting KTD and GPTD to deep RL tasks.  

Our main contributions are: (1) we propose an EKF-based framework for policy evaluation in RL, which accounts for both parameter and observation uncertainties; (2) we devise an optimization method, called KOVA, and prove it minimizes at each time step a regularized objective function. This optimizer can be easily incorporated as a policy evaluation component in popular policy optimization algorithms, and improve upon existing methods; (3) we clarify the differences between KOVA vs. GPTD and KTD and explain why KOVA adapts them to deep RL domains; (4) we demonstrate the improvement achieved by KOVA on several control tasks with continuous state and action spaces. 

\begin{table*}[t]
	\small
	\centering
	\caption{Examples for policy optimization algorithms and their Bellman TD error  $\delta(u; \boldsymbol{\theta}_t)$ type. The decomposition of $\delta(u; \boldsymbol{\theta}_t)$ into the observation function $h(u; \boldsymbol{\theta}_t)$ and the target label $y(u)$ in the EKF model (\ref{eq:Extended-Kalman}) enables the integration of our KOVA optimizer with policy optimization algorithms that include a policy evaluation phase. $\boldsymbol{\theta}'$ refers to a fixed parameter, different from the one being trained $\boldsymbol{\theta}_t$.}
	\begin{tabular}{|p{39mm}|p{30mm}|p{18mm}|p{26mm}|p{44mm}|}
		\hline
		{\it Algorithm} &  {\it Bellman TD error} \newline $\delta(u; \boldsymbol{\theta}_t)$ & {\it Transition} \newline {\it input} $u$ &{\it Observation function} \newline {\bf $h(u; \boldsymbol{\theta}_t)$ }&  {\it Target label} $y(u)$ \\ \hline \hline
		DQN \cite{mnih2013playing} & Optimality equation & $(s, a, r, s')$ & $Q(s, a; \boldsymbol{\theta}_t)$ & $r + \gamma \max_{a'} Q(s', a'; \boldsymbol{\theta}')$ \\ \hline
		DDPG \cite{lillicrap2015continuous} & $1$-step Q-evaluation & $(s, a, r, s')$ & $Q(s, a; \boldsymbol{\theta}_t)$ & $r + \gamma Q(s', \pi(s'); \boldsymbol{\theta}')$ \\ \hline
		A3C  \cite{mnih2016asynchronous}   &  $k$-step V-evaluation & $(s_m, r_m,\!\ldots,$ \newline $r_{m\!+\!k\!-1},  s_{m\!+\!k})$ & $V(s_m; \boldsymbol{\theta}_t)$  &  $\sum_{i=0}^{k-1}\gamma^i r_{m+i} + \gamma^k  V(s_{m + k}; \boldsymbol{\theta}')$ \\ \hline
		 TRPO \cite{schulman2015trust} \newline PPO \cite{schulman2017proximal}  & GAE \newline \cite{schulman2015high} & $(s_m, r_m,\!\ldots,$ \newline $r_{m\!+\!k\!-1},  s_{m\!+\!k})$ & $V(s_m; \boldsymbol{\theta}_t)$ & 
		$\sum_{i=0}^{\infty}(\gamma \lambda)^i \big(r_{m\!+\!i}\!+\!\gamma  V(s_{m\!+\!i\!+\!1}; \boldsymbol{\theta}') $ \newline $ - V(s_{m + i}; \boldsymbol{\theta}') \big)  + V(s_{m}; \boldsymbol{\theta}')$ \\ \hline
		SAC \cite{haarnoja2018soft} & $1$-step V-evaluation \newline $1$-step Q-evaluation &  $(s,a)$  \newline $(s, a, r)$ & $V(s; \boldsymbol{\theta}_t)$ \newline
		$Q(s, a; \tilde{\boldsymbol{\theta}}_t)$ & $\mathbb{E}_{a}[Q(s, a; \tilde{\boldsymbol{\theta}}_t) - \log \pi(a|s)]$ \newline $r + \gamma \mathbb{E}_{s'}[V(s';\boldsymbol{\theta}')]$\\ \hline
	\end{tabular}
	\label{vf_table}
\end{table*}

\section{Background}
\subsection{Extended Kalman Filter (EKF)}
\label{Section:EKF}
We briefly outline EKF \cite{anderson1979optimal, gelb1974applied}, that serves as a foundation of the formulation we present in this work. The EKF can be used for learning parameters $\boldsymbol{\theta}$ of a non-linear approximation function. It considers the following model:
\begin{equation}
\label{eq:Extended-Kalman}
\begin{cases}
\boldsymbol{\theta}_t = \boldsymbol{\theta}_{t-1} + {\bf v}_t\\
y({\bf u}_t) = h({\bf u}_t; \boldsymbol{\theta}_t) + {\bf n}_t
\end{cases},
\end{equation}
where $\boldsymbol{\theta}_t\!\in\!\mathbb{R}^{d\!\times\!1}$ is a parameter vector evaluated at time $t$, ${\bf v}_t$ is the evolution noise, ${\bf n}_t$ is the observation noise, both modeled as additive white noises with covariances ${\bm P}_{{\bm v}_t}$ and ${\bm P}_{{\bm n}_t}$, respectively. The {\em observations vector} is denoted as $y({\bf u}_t)\!=\![y(u_t^1),\!y(u_t^2),\!\ldots,\!y(u_t^N)]^\top\!\in\!\mathbb{R}^{N\!\times\!1}$ for some input vector ${\bf u}_t\!\in\!\mathbb{R}^{N\!\times\!1}$ at time $t$. The vectorial {\em observation function} is composed of $N$ non-linear functions $h(u_t^i; \boldsymbol{\theta}_t)$: $h({\bf u}_t; \boldsymbol{\theta}_t)\!=\![
h(u_t^1; \boldsymbol{\theta}_t), 
h(u_t^2; \boldsymbol{\theta}_t), 
\ldots,
h(u_t^N; \boldsymbol{\theta}_t)  
]^\top$. 

As seen in (\ref{eq:Extended-Kalman}), EKF treats the parameters $\boldsymbol{\theta}_t$ as {\em RVs}, similarly to Bayesian approaches that assume the parameters to belong to an uncertainty set $\Theta$ governed by the mean and covariance of the parameters distribution. The estimation at time t, denoted as $\boldsymbol{\hat{\theta}}_{t|\cdot}$ is the conditional expectation of the parameters with respect to the observed data: $\boldsymbol{\hat{\theta}}_{t|t}  \triangleq \mathbb{E} [ \boldsymbol{\theta}_t | y_{1:t}]$ and  $\boldsymbol{\hat{\theta}}_{t|t-1}  \triangleq \mathbb{E} [ \boldsymbol{\theta}_t | y_{1:t-1}]  = \boldsymbol{\hat{\theta}}_{t-1|t-1}$. With some abuse of notation, $y_{1:t'}$ are the observations gathered up to time $t'$: $y ({\bf u}_1), \ldots, y ({\bf u}_{t'})$. The {\it parameters error} is $\boldsymbol{\tilde{\theta}}_{t|\cdot}  \triangleq \boldsymbol{\theta}_t - \boldsymbol{\hat{\theta}}_{t|\cdot}$ and the conditional {\it error covariance} is ${\bf P}_{t|\cdot}  \triangleq \mathbb{E} \big[  \boldsymbol{\tilde{\theta}}_{t|\cdot}  \boldsymbol{\tilde{\theta}}_{t|\cdot}^\top | y_{1:\cdot}\big]$. EKF further defines the {\it observation innovation}, the {\it covariance of the innovation} and the {\it Kalman gain} in Equations (\ref{eq:The observation innovation}) - (\ref{eq:Kal_gain}) respectively. Then, it updates the parameter and covariance according to Equation (\ref{eq:EKF}):
\begin{align}
\label{eq:The observation innovation} & {\bf \tilde{y}}_{t|t-1} \triangleq h({\bf u}_t; \boldsymbol{\theta}_t) - \mathbb{E}[h({\bf u}_t; \boldsymbol{\theta}_t)|y_{1:t-1}]\\
\label{eq:The covariance of the innovation} & {\bf P}_{{\bf \tilde{y}}_t}  \triangleq \mathbb{E}[ {\bf \tilde{y}}_{t|t-1} {\bf \tilde{y}}_{t|t-1}^\top |y_{1:t-1}] + {\bf P}_{{\bf n}_t} \\
\label{eq:Kal_gain} & {\bf K}_t \triangleq \mathbb{E}[ \boldsymbol{\tilde{\theta}}_{t|t-1} {\bf \tilde{y}}_{t|t-1} |y_{1:t-1}] {\bf P}_{{\bf \tilde{y}}_t}^{-1}\\
& \label{eq:EKF} \begin{cases}
\boldsymbol{\hat{\theta}}_{t|t}^{\text{EKF}} = \boldsymbol{\hat{\theta}}_{t|t-1} +  {\bf K}_t \big( y({\bf u}_t) - h({\bf u}_t; \boldsymbol{\hat{\theta}}_{t|t-1}) \big),\\
{\bf P}_{t|t} = {\bf P}_{t|t-1} - {\bf K}_t {\bf P}_{{\bf \tilde{y}}_t} {\bf K}_t^\top.
\end{cases}
\end{align}

\subsection{Reinforcement Learning and MDPs}
\label{Reinforcement learning}
The standard RL setting considers interaction of an agent with an environment $\mathcal{E}$. The environment is modeled as a Markov Decision Process (MDP) $\{\mathcal{S}, \mathcal{A}, P, R, \gamma\}$ where $\mathcal{S}$ is a set of states, $\mathcal{A}$ is a set of actions, $P:\!\mathcal{S}\!\times\!\mathcal{A}\rightarrow\![0,1]^{|\mathcal{S}|}$ is the state transition probabilities for each state $s$ and action $a$, $R:\!\mathcal{S}\!\times\!\mathcal{A}\!\rightarrow\!\mathbb{R}$ is a reward function and $\gamma\!\in\![0,1)$ is a discount factor. At each time step $t$, the agent observes state $s_t\!\in\!\mathcal{S}$ and chooses action $a_t\!\in\!\mathcal{A}$ according to a policy $\pi:\!\mathcal{S}\!\times\!\mathcal{A}\!\rightarrow\![0,1]$. The agent receives an immediate reward $r_t(s_t, a_t)$ and the environment stochastically steps to state $s_{t+1}\!\in\!\mathcal{S}$ according to the probability distribution  $P(s_{t+1}|s_t, a_t)$. The state-value function and the state-action Q-function are used for evaluating the performance of a fixed policy $\pi$ \cite{sutton1998reinforcement}: $V^{\pi}(s)\!=\!\mathbb{E}^{\pi}\big[ \sum_{t=0}^{\infty} \gamma^t r(s_t, a_t) | s_0 = s \big]$ and $Q^{\pi}(s, a)\!=\!\mathbb{E}^{\pi}\big[ \sum_{t=0}^{\infty} \gamma^t r_t(s_t, a_t)  | s_0 = s, a_0 = a \big]$. These functions satisfy a recursion form defined by the following Bellman equations: $V^{\pi}(s)\!=\!\mathbb{E}^{\pi}\big[ r(s, a)\!+\!\gamma V^{\pi}(s') \big]$ and $Q^{\pi}(s, a)\!=\!\mathbb{E}^{\pi}\big[ r(s, a)\!+\!\gamma Q^{\pi}(s', a')\big]$, where $s'$ is the successive state after $s$, $a' \sim \pi(s')$ and the expectations are with respect to the state (state-action) distribution induced by transition law $P$ and policy $\pi$. 

\subsection{Value-based Function Estimation}
In this paper, we will use the term {\em value-based function} (VBF) to denote $V^{\pi}(s)$ and $Q^{\pi}(s,\!a)$. When the state or action space is large, the VBF can be approximated using a parameterized function, $h(\cdot; \boldsymbol{\theta})$. We focus on general, possibly non-linear approximations such as DNNs that can effectively learn complex approximations. We also use the notation $u$ to specify the {\it transition input} for the approximated VBF at time $t$, $h(u; \boldsymbol{\theta}_t)$, and for a {\it target label} $y(u)$. For example, for $h(u; \boldsymbol{\theta}_t)\!=\!V^{\pi}(s_m; \boldsymbol{\theta}_t)$, and $y(u)\!=\!r_m\!+\!\gamma V^{\pi}(s_{m+1}; \boldsymbol{\theta}')$,  $u\!=\!(s_m, r_m, a_m, s_{m+1})$ is the 1-step transition at time $m$.  In Table \ref{vf_table}, we provide examples of several options for $y(u)$ and $h(u; \boldsymbol{\theta}_t)$ which clarify how this general notation can be utilized in some common policy optimization algorithms. Note that $y(u)$ may exploit information from successive transitions since $u$ may contain multiple transitions.

Extensive research has been dedicated to learn approximated VBFs \cite{dann2014policy,maei2011gradient}, mostly by optimizing VBF parameters at each time step $t$ through minimization of the empirical mean squared {\it Bellman TD error}  $\delta (u; \boldsymbol{\theta}_t) \triangleq y(u)\!-\!h(u; \boldsymbol{\theta}_t)$, over a batch of $N$ samples generated from environment $\mathcal{E}$ under a given policy:
\begin{equation}
\label{eq:vf objective}
L_t^{\text{MLE}}(\boldsymbol{\theta}_t) = \frac{1}{2N} \sum_{i=1}^{N} \delta^2 (u_i; \boldsymbol{\theta}_t).
\end{equation} 
Traditionally, VBFs are trained by gradient methods: $\boldsymbol{\theta}_{t+1}\!\leftarrow\!\boldsymbol{\theta}_t + \alpha \mathbb{E}_{u \sim p(\cdot)} \big[ \big(y(u)\!-\!h(u; \boldsymbol{\theta}_t)\big) \nabla_{\boldsymbol{\theta}_t} h(u; \boldsymbol{\theta}_t) \big]$, where $\alpha$ is the learning rate and $p(\cdot)$ is the experience distribution. This training procedure seeks for a {\em point-estimate} of the parameters. 

\subsection{Kalman Approach for VBFs Estimation}
\label{Section:KTD}
Optimizing VBFs by looking for the parameter point-estimate, lacks a broader perspective of possible VBF uncertainties. This motivated \citet{engel2003bayes,engel2005reinforcement} to treat the value function and the observed rewards as Gaussian processes. In their algorithm, called GPTD, the VBF is a {\em random variable}. Although the value is no longer deterministic, we will keep using the term VBF for $V^{\pi}$ and $Q^{\pi}$ in order to be clear when comparing to point-estimate methods. 
\citet{geist2010kalman} extended GPTD to non-stationary environments and to parametric non-linear approximations. They proposed an algorithm called Kalman Temporal Difference (KTD) which uses the Unscented Kalman filter (UKF) \cite{julier1997new,wan2000unscented} to learn the uncertainty in VBF parameters. The formulations proposed in KTD and GPTD are appealing since they model both observations and parameter uncertainties during  estimation, which in turn improve the learned policy and encourage exploration \cite{geist2010kalman}.

Surprisingly, we did not see any attempt to apply these algorithms to deep RL. When looking at the challenges of policy evaluation in deep RL tasks, as we described in the Introduction, KTD, GPTD and their subsequent works \cite{pietquin2011sample,wang2014kalman,kitao2017model,ghavamzadeh2016bayesian} suffer from some major limitations:\\
{\bf 1. ``Off policy issue''}: In KTD and GPTD the observation function being optimized is the {\em difference} between VBFs of two successive states. When transitions are stochastic, the observation noise is colored, resulting a memory effect of previous transitions. This phenomenon prevents KTD and GPTD from performing in off-policy settings (see \cite{geist2010kalman} for more details)\footnote{GPTD can handle the colored noise in linear parametrization, however, for non-linear cases, this phenomenon still exists.}.\\ 
{\bf 2. ``Coupling issue''}: In KTD, the observation $r_t$ at time $t$ is a scalar. This formulation {\em couples} the parameters update index $t$ in $\boldsymbol{\theta}_t$ to the transition index in $(s_t, a_t, r_t, s_{t+1})$. This means that in KTD the parameters are updated with {\em every single transition} performed by the RL agent, enforcing training batch of size 1. \\
{\bf 3. ``Growing vector-size issue''}: According to GPTD formulation, the observations vector contains successive transitions from a single trajectory. This prevents it from using independent transitions from possibly different trajectories in a single batch. Moreover, the size of the batch vector grows with every new observed transition, resulting a varying batch size, which can cause computations issues when optimizing DNNs.\\
{\bf 4. ``Excessive computations issue''}: Through its formalization as UKF, KTD involves Cholesky decomposition applied on a $d\!\times\!d$ matrix and requires running $2\!*\!(2d\!+\!1)$ forward passes in each optimization step (see Figure \ref{fig:kova_vs_ktd}). Recall that $d$ denotes the parameter dimension. When we consider DNNs, these requirements are a major drawback which prevents KTD from being applicable to deep RL domains. 

These limitations motivate us to formulate a Kalman-based framework for approximating VBFs that would enjoy the advantages of Kalman optimizers while addressing the challenges of policy evaluation in deep RL tasks. In the next section, we discuss how adopting the EKF model can circumvent these limitations, and how we change the formulation of the model in order to adapt it to deep RL tasks. 

\begin{figure}[t]
	\centering{
		\includegraphics[width=1.0\linewidth,keepaspectratio]{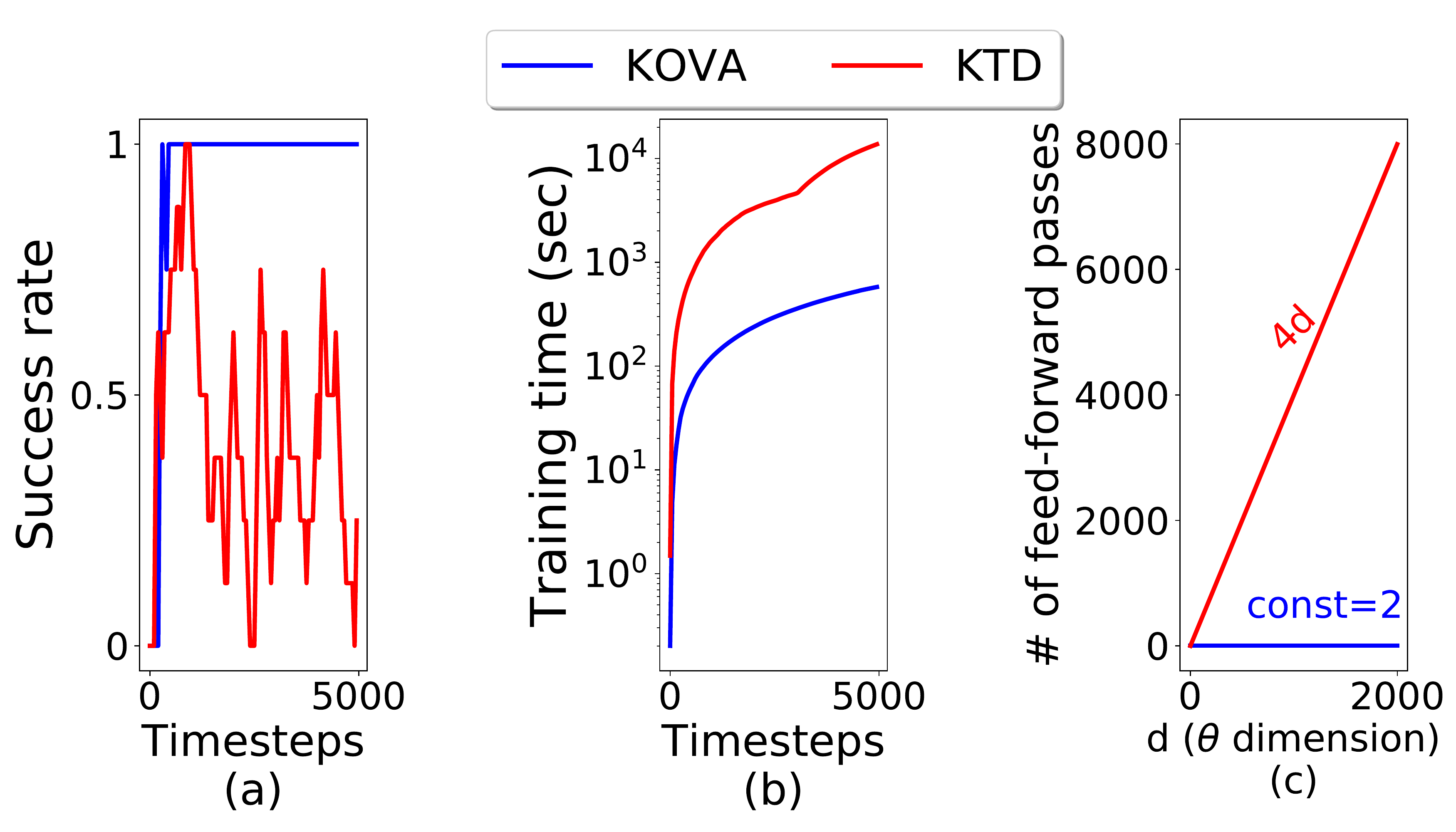}} 
	\caption{{\bf Motivating example} (a) The success rate of KOVA vs KTD for solving a simple $4\times4$ maze. The Q-function is modeled as a DNN with $d=340$ parameters. KOVA quickly solves the maze, while KTD fails. (b) Excessive computations required by KTD significantly slow training. (c) While KOVA requires only two DNN feed-forward passes in each optimization step, KTD requires $\sim\!4d$ passes. As $d$ grows, it becomes impossible to apply KTD in deep RL tasks.} 
	\label{fig:kova_vs_ktd}
\end{figure}

\section{KOVA for VBF Approximation}
\label{Section:EKF_RL}
Our goal is to estimate the VBFs parameters $\boldsymbol{\theta}_t$. One way is to learn them by maximum likelihood estimation (MLE) using stochastic gradient descent methods: $\boldsymbol{\theta}^{\text{MLE}} = \arg\max_{\boldsymbol{\theta}} \log p(y_{1:t}|\boldsymbol{\theta})$. This forms the objective function in Equation (\ref{eq:vf objective}). The MLE approach assumes that the parameters are deterministic. We propose a different approach for approximating VBFs: Assume $\boldsymbol{\theta}$ is a parameter vector modeled as a random process. In each time step $t$, $\boldsymbol{\theta}$ is an {\em RV} with prior $p(\boldsymbol{\theta})$. Applying Bayes rule, we can learn the parameters by calculating the maximum a-posteriori (MAP) estimator: 
$\boldsymbol{\theta}^{\text{MAP}}\!=\!\arg\max_{\boldsymbol{\theta}} \{\log p(\boldsymbol{\theta}|y_{1:t})\!=\!\arg\max_{\boldsymbol{\theta}} \log p(y_{1:t}|\boldsymbol{\theta}) + \log p(\boldsymbol{\theta}) \}$. We will use the iterative form of $\boldsymbol{\theta}_t^{\text{MAP}}$:   
\begin{align}
\label{eq:MAP_EKF}
\boldsymbol{\theta}_t^{\text{MAP}} = \arg\max_{\boldsymbol{\theta}_t} \{\log p(y_t|\boldsymbol{\theta}_t) + \log  p(\boldsymbol{\theta}_t|y_{1:t-1}) \}.
\end{align}
Here, instead of using the parameters prior, we use an equivalent derivation for the parameters posterior conditioned on $y_{1:t}$, based on the likelihood of a {\it single vector } of observations $y_t \triangleq y({\bf u}_t)$ and the posterior conditioned on $y_{1:t-1}$ \cite{van2004sigma}. This unique derivation of the MAP estimator is a key step for the vectorial incremental Kalman updates we present later in our KOVA algorithm and for further defining $L_t^{\text{EKF}}$ (\ref{eq:EKF-loss}). 

This alternative approach is modeled in Equation (\ref{eq:Extended-Kalman}). EKF assumes two sources of randomness in VBFs: The first equation in  (\ref{eq:Extended-Kalman}) models the parameters vector $\boldsymbol{\theta}_t$ as an RV. The second equation in (\ref{eq:Extended-Kalman}) models the observations vector $y({\bf u}_t)$ as an RV. These two uncertainty models are {\em independent} of each other. The first is the ``epistemic'' uncertainty which may model our uncertainty in the MDP model (state transition probabilities and reward function). The second is the ``aleatoric'' uncertainty which is related to the stochastic transitions in the trajectory and to the possibly random policy. The role and effect of these two distinguish uncertainties is carefully discussed by \citet{osband2018randomized}.

In order so solve the MAP estimation in Equation (\ref{eq:MAP_EKF}) we need to define the likelihood  $p(y_{t}|\boldsymbol{\theta}_t)$ and the posterior $p(\boldsymbol{\theta}_t|y_{1:t-1})$. We adopt the EKF model (\ref{eq:Extended-Kalman}), and make the following assumptions: 
\begin{assumption}
	\label{As:ConditionalIndependance2}
	The likelihood $p(y({\bf u}_t)|\boldsymbol{\theta}_t)$ is Gaussian: 
	$y({\bf u}_t)|\boldsymbol{\theta}_t \sim \mathcal{N}( h({\bf u}_t; \boldsymbol{\theta}_t), {\bf P}_{{\bf n}_t})$. 
\end{assumption}
\begin{assumption}
	\label{As:GaussianPosterior2}
	The posterior distribution $p(\boldsymbol{\theta}_t|y_{1:t-1})$ is Gaussian: $\boldsymbol{\theta}_t|y_{1:t-1} \sim \mathcal{N}(\boldsymbol{\hat{\theta}}_{t|t-1},{\bf P}_{t|t-1})$.
\end{assumption}
These assumptions are common when using EKF. In the context of RL, these assumptions add flexibility: the VBFs $h({\bf u}_t; \boldsymbol{\theta}_t)$ are treated as {\em RVs} and information is gathered on the uncertainty of their estimate. In addition, the noisy observations (target labels) can have different variances and can even be correlated. 

Before introducing Theorem \ref{theorem1}, we define the {\em Bellman TD error vector} as $\delta({\bf u}_t; \boldsymbol{\theta}_t)=\![\delta(u_t^1; \boldsymbol{\theta}_t), .., \delta(u_t^N; \boldsymbol{\theta}_t)]^\top=\!y({\bf u}_t) - h({\bf u}_t; \boldsymbol{\theta}_t)$. We also define the first order Taylor expansion for the observation function $h(\boldsymbol{\theta}_t)$: $h({\bf u}_t; \boldsymbol{\theta}_t) 
= h({\bf u}_t; \boldsymbol{\hat{\theta}})\!+\!  \nabla_{\boldsymbol{\theta}_t} h({\bf u}_t; \boldsymbol{\hat{\theta}})^\top \big( \boldsymbol{\theta}_{t}\!-\!\boldsymbol{\hat{\theta}} \big)$, where 
$ \nabla_{\boldsymbol{\theta}_t} h({\bf u}_t; \boldsymbol{\hat{\theta}}) 
= \begin{bmatrix} 
\nabla_{\boldsymbol{\theta}_t} h(u_t^1; \boldsymbol{\hat{\theta}}) , \nabla_{\boldsymbol{\theta}_t} h(u_t^2; \boldsymbol{\hat{\theta}}), \ldots , \nabla_{\boldsymbol{\theta}_t} h(u_t^N; \boldsymbol{\hat{\theta}})  \end{bmatrix} \in \mathbb{R}^{d \times N} $
and $\boldsymbol{\hat{\theta}}$ is typically chosen to be the previous estimation of the parameters at time $t\!-\!1$,  $\boldsymbol{\hat{\theta}}=\boldsymbol{\hat{\theta}}_{t|t-1}$. We can now derive a regularized objective function, $L_t^{\text{EXF}}$ (\ref{eq:EKF-loss}), and argue in its favor for optimizing VBFs in RL:

\begin{theorem}
	\label{theorem1}
	Under Assumptions \ref{As:ConditionalIndependance2} and \ref{As:GaussianPosterior2}, $\boldsymbol{\hat{\theta}}^{\text{EKF}}_{t|t}$  (\ref{eq:EKF}) minimizes at each time step $t$ the following regularized objective function: 
	\begin{align}
	\label{eq:EKF-loss}
	\nonumber L^{\text{EKF}}_t(\boldsymbol{\theta}_t) & =  \frac{1}{2}  \delta({\bf u}_t; \boldsymbol{\theta}_t)^\top {\bf P}_{{\bf n}_t}^{-1}  \delta({\bf u}_t; \boldsymbol{\theta}_t) \\
	& +  \frac{1}{2}(\boldsymbol{\theta}_t - \boldsymbol{\hat{\theta}}_{t|t-1})^\top {\bf P}_{t|t-1}^{-1} (\boldsymbol{\theta}_t - \boldsymbol{\hat{\theta}}_{t|t-1}).
	\end{align}
\end{theorem}

The proof for Theorem \ref{theorem1} appears in the supplementary material. It is based on the following key steps: (i) Solving the maximization problem in (\ref{eq:MAP_EKF}) using the EKF model (\ref{eq:Extended-Kalman}) combined with the Gaussian Assumptions \ref{As:ConditionalIndependance2} and \ref{As:GaussianPosterior2}; (ii) Using first order Taylor series linearization for the VBF $h({\bf u}_t; \boldsymbol{\theta}_t)$. 

Next, we explicitly write the expressions for the statistics of interest in Equations (\ref{eq:The observation innovation})-(\ref{eq:Kal_gain}). {\it The covariance of the innovation} and the {\em Kalman gain} become:
\begin{align}
\label{eq:statistics} {\bf P}_{{\bf \tilde{y}}_t}  & = \nabla_{\boldsymbol{\theta}_t} h({\bf u}_t; \boldsymbol{\hat{\theta}})^\top {\bf P}_{t|t-1} \nabla_{\boldsymbol{\theta}_t} h({\bf u}_t; \boldsymbol{\hat{\theta}})\!+\!{\bf P}_{{\bf n}_t}.\\
\label{eq:Kalman_gain}
{\bf K}_t & = {\bf P}_{t|t-1} \nabla_{\boldsymbol{\theta}_t} h({\bf u}_t,  \boldsymbol{\hat{\theta}})  {\bf P}_{{\bf \tilde{y}}_t}^{-1}.
\end{align}
This Kalman gain is used in the parameters update and the error covariance update in Equation (\ref{eq:EKF}). We can see that the Kalman gain propagates new information from the noisy target labels, back down into the parameters uncertainty set $\Theta$, before combining it with the estimated parameter value. The gain ${\bf K}_t$ can be interpreted as an adaptive learning rate for each individual parameter that implicitly incorporates the uncertainty of each parameter. This approach resembles familiar stochastic gradient optimizers such as Adagrad \cite{duchi2011adaptive}, AdaDelta \cite{zeiler2012adadelta}, RMSprop \cite{tieleman2012lecture}
and Adam \cite{kingma2014adam}, for different choices of ${\bf P}_{t|t-1}$ and ${\bf P}_{{\bf n}_t}$. \citet{ruder2016overview} discussed and compared these optimizers. 

The regularization term in $L_t^{\text{EKF}}$ can be seen as imposing a {\em trust-region} on the VBF parameters, similarly to trust-region over policy parameters in policy optimization methods \cite{schulman2015trust}.
We prove this trust-region property in the supplementary material. 

\subsection{Comparing $L_t^{\text{EKF}}$ and $L_t^{\text{MLE}}$ for Optimizing VBFs}
\label{Section:comparison}
We argue in favor of using $L^{\text{EKF}}_t(\boldsymbol{\theta}_t)$ (\ref{eq:EKF-loss}) for optimizing VBFs instead of the commonly used $L^{\text{MLE}}_t(\boldsymbol{\theta}_t)$ (\ref{eq:vf objective}).
The connection between these two objective functions can be summarized by the following Corollary:
\begin{corollary}
	\label{corollary1}
	Under Assumptions \ref{As:ConditionalIndependance2} and \ref{As:GaussianPosterior2}, consider a diagonal covariance ${\bf P}_{{\bf n}_t}$ with diagonal elements $\sigma_i = N$  and assume ${\bf P}_{0|0} = {\bf P}_{{\bf v}_t} = {\bf 0}$, then:
	$L^{\text{EKF}}_t(\boldsymbol{\theta}_t) = L^{\text{MLE}}_t(\boldsymbol{\theta}_t)$.
\end{corollary}
The proof is given in the supplementary material. We can see that the two objectives are the same if we assume that the parameters are deterministic and that the noisy target labels have a fixed variance. We conclude that $L^{\text{MLE}}_t$ does not consider two types of uncertainties: (i) parameters uncertainty since it lacks regularization with ${\bf P}_{t|t-1}$. Note that when adding a standard $L_2$ regularization to $L^{\text{MLE}}_t$, often common in DNNs, it favors staying close to a prior assumption of $\boldsymbol{\theta}\!\sim ({\bf 0}, {\bf I})$ but without updating the prior. (ii) observations uncertainty since it lacks regularization with ${\bf P}_{{\bf n}_t}$ which reflect the amount of confidence we have in the observations. We further discuss these uncertainties in Section \ref{Sec:uncertainties}. 

\begin{algorithm}[t] 
	{\small 
		\caption{KOVA Optimizer}
		\label{alg:KOVA}
		\begin{algorithmic}[1]
			\REQUIRE ${\bf P}_{0|0}$, ${\bf P}_{{\bf v}_t}$,  ${\bf P}_{{\bf n}_t}$, $\alpha$,  $\mathcal{R}$. \ \ {\bf Initialize: } $\boldsymbol{\hat{\theta}}_{0|0}$.
			\FOR {$t=1, \ldots, T$}
			\STATE Set predictions: $\begin{cases}
			\boldsymbol{\hat{\theta}} = \boldsymbol{\hat{\theta}}_{t|t-1} = \boldsymbol{\hat{\theta}}_{t-1|t-1}\\
			{\bf P}_{t|t-1} = {\bf P}_{t-1|t-1} + {\bf P}_{{\bf v}_t}
			\end{cases}$
			\STATE Sample N tuples $\{y(u^i), h(u^i; \boldsymbol{\hat{\theta}})\}_{i=1}^N$ from $\mathcal{R}$.
			\STATE Construct $N$-dim vectors $y({\bf u}_t)$ and $ h({\bf u}_t,\boldsymbol{\hat{\theta}})$. 
			\STATE Compute $(d \times N)$-dim matrix $\nabla_{\boldsymbol{\theta}} h({\bf u}_t; \boldsymbol{\hat{\theta}})$.  
			\STATE Compute ${\bf P}_{{\bf \tilde{y}_t}}$ (\ref{eq:statistics}) and ${\bf K}_t$ (\ref{eq:Kalman_gain}).
			\STATE Set updates: \\
			$\begin{cases}
			\boldsymbol{\hat{\theta}}_{t|t} = \boldsymbol{\hat{\theta}}_{t|t-1} + \alpha {\bf K}_t \big( y({\bf u}_t) - h({\bf u}_t; \boldsymbol{\hat{\theta}}_{t|t-1}) \big)\\
			{\bf P}_{t|t} = {\bf P}_{t|t-1} - \alpha {\bf K}_t {\bf P}_{{\bf \tilde{y}}_t} {\bf K}_t^\top
			\end{cases}$
			\ENDFOR
			\ENSURE $\boldsymbol{\hat{\theta}}_{t|t}$ and ${\bf P}_{t|t}$
		\end{algorithmic}
	}
\end{algorithm}

\subsection{Practical algorithm: KOVA optimizer} 
\label{section:KOVA}
We now present the KOVA optimizer in Algorithm \ref{alg:KOVA} for approximating VBFs. KOVA calculates at each optimization step the parameter estimate $\boldsymbol{\hat{\theta}}_{t|t}$ and the conditional error covariance ${\bf P}_{t|t}$. Notice that $\mathcal{R}$ is a sample generator whose structure depends on the policy algorithm for which KOVA is used as a policy evaluation component. $\mathcal{R}$ can contain trajectories from a fixed policy or it can be an experience buffer which contains transitions from several different policies.

The estimated parameter $\boldsymbol{\hat{\theta}}_{t|t}$ is a minimizer of $L^{\text{EKF}}_t$ (\ref{eq:EKF-loss}). Recall that $d$ is the dimension of the parameters and $N$ is the batch size.  Directly optimizing $L^{\text{EKF}}_t$ is hard since it requires inversing the $(d\!\times\!d)$-dimensional matrix ${\bf P}_{t|t-1}$. Alternatively, we use in practice the update Equations (\ref{eq:EKF}) and the Kalman gain equations in (\ref{eq:statistics})-(\ref{eq:Kalman_gain}) in order to avoid this inversion. In addition, we add a fixed learning rate $\alpha$ to smooth the update.

{\bf Algorithm complexity:} For a $d$-dimensional parameter vector $\boldsymbol{\theta}$, our algorithm requires $\mathcal{O}(d^2)$  extra space to store the covariance matrix and $\mathcal{O}(d^2)$ computations for matrix multiplications. This complexity is typical to second-order optimization methods. However, our update method does not require inverting the $\small (d\!\times\!d)$ matrix ${\bf P}_{t|t-1}$ in the update process, but only requires inverting the  $\small (N\!\times\!N)$ matrix $\small \big(\nabla h(\boldsymbol{\hat{\theta}})^\top {\bf P}_{t|t-1} \nabla h( \boldsymbol{\hat{\theta}})\!+\!{\bf P}_{{\bf n}_t} \big)^{-1}$. Usually, $ {\small N\!\ll\!d}$. The extra time and memory requirements can be tolerated for small networks of size $d$. However, they can be considered as drawbacks for large network sizes. Nevertheless, there are several options for relaxing these drawbacks: {\bf (a)} GPUs for matrix multiplications can accelerate the computation time. {\bf (b)} We can assume correlations only between parameters in the same DNN layer and apply layer factorization. This can significantly reduce the computation cost and memory requirements \cite{puskorius1991decoupled,zhang2017learning,wu2017scalable}. {\bf (c)} We can apply KOVA only on the last layer in large DNNs, similarly to \citet{levine2017shallow} who optimized the last layer using linear least squares optimization methods. Yet, our approach scales to continuous state and action spaces, e.g. in control problems. 

\subsection{Comparing KOVA to KTD and GPTD}

In this section we highlight the adaptations we made to the models proposed in KTD and GPTD in order to meet the challenges of policy evaluation in deep RL tasks. The integration of all these important adaptations ensures that KOVA can be incorporated as the policy evaluation component in other common policy optimization algorithms. \\ 
{\bf 1. Decomposing the Bellman TD error:} The main idea is to decompose the Bellman TD error vector $\delta({\bf u}_t; \boldsymbol{\theta}_t)$ into two parts: 
$\delta({\bf u}_t; \boldsymbol{\theta}_t)\!=\!y({\bf u}_t)\!-\!h({\bf u}_t; \boldsymbol{\theta}_t)$. The first part is the {\it target-label vector} at time $t$, $y({\bf u}_t)$, which contains $N$ target labels as described in Section \ref{Section:EKF}. The second part is the {\it observation function vector} $h({\bf u}_t; \boldsymbol{\theta}_t)$ which contains $N$ VBFs $h(u; \boldsymbol{\theta}_t)$ for $N$ different inputs. 
In Table \ref{vf_table} we provide several examples for the Bellman TD error decomposition according to the chosen policy optimization algorithm.\\
\underline{{\em Comparison to KTD /\ GPTD}}: In KTD and GPTD formulations, the observation function is the {\em differnece} between VBFs of two succesive states. In KOVA, by decomposing $\delta({\bf u}_t; \boldsymbol{\theta}_t)$, the observation function if only a VBF of a single state, while the VBF of the successive state is considered in the target label. Therefore, the observation function is {\em differential}, allowing us to use first order Taylor expansion linearization and observation noise can now be assumed white and not colored. This enables KOVA to learn in both off-policy and on-policy manners, and solves the ``off policy issue'' we discussed above. In addition it enables KOVA to use target parameters $\boldsymbol{\theta}'$ for $h(s'; \boldsymbol{\theta}')$ instead of trainable parameters for stabilizing the training process.\\ 
{\bf 2. Vectorization of observations:} Our formulation enables handling mini-batches through vectorization of $y({\bf u}_t)$ and $ h({\bf u}_t; \boldsymbol{\theta}_t)$. This decouples the connection between the update index $t$ and the transition index in the trajectory: observations of different inputs from different trajectories can be considered in each update time. This vectorization is {\em nontrivial}, since observation vectors in successive time steps are {\em input-independent}, unlike classic formulations of EKF. \\ 
\underline{{\em Comparison to KTD /\ GPTD}}: The vectorization of observations solves the ``coupling issue'' and allows to use $N$ transition samples from trajectories as opposed to a single transition in KTD in every optimization step $t$. It also solves the ``growing vector size  issue'': $y({\bf u})$ has a fixed size $N$ as opposed to the growing size vectors in GPTD. This facilitates the training process with DNNs as required by deep RL algorithms. In addition, this vectorization enables integrating both correlated and i.i.d transitions through ${\bf P}_{{\bf n}_t}$.\\  
{\bf 3. EKF vs UKF:} The use of EKF in KOVA dramatically reduces the computations compared to UKF in KTD. Based on first order Taylor linearization of the VBF $h(\cdot; \boldsymbol{\theta})$, KOVA only requires computation of the VBF gradient, solving the problems arised in ``excessive computations issue''.   

\subsection{Uncertainties in KOVA}
\label{Sec:uncertainties}

In KOVA, the hyper parameters are incorporated in the covariances ${\bf P}_{{\bf v}_t}$ and ${\bf P}_{{\bf n}_t}$, which means the hyper-parameters have a meaning. We now explain some guidelines for choosing the values and structures of these matrices.  

{\bf ${\bf P}_{{\bf v}_t}$:} This is the evolution noise covariance matrix which models a possibly non stationary parameters of the VBF and adds a {\em tracking} property to KOVA. When the policy that generates the samples is changing, the associated VBF changes too \cite{geist2010kalman}. Adding evolution noise proportional to conditional error covariance at each step (for example ${\bf P}_{{\bf v}_t} = \frac{\eta}{1-\eta} {\bf P}_{t-1|t-1}$ for some small values of $\eta$) corresponds to a decay factor for the weight of previous observations \cite{ollivier2018online}, which stabilize the estimating process. 

{\bf ${\bf P}_{{\bf n}_t}$:} This is the covariance matrix of vectorized target labels. If the mini-batch of transitions are sampled independently, then this matrix should be diagonal. KOVA general framework offers two possible extensions: (i) keeping the independence assumption of the observations, however setting different variances to each noisy observation. The estimated TD variances as suggested by \citet{sherstan2018directly} may be considered here. (ii) if the transition samples are taken as successive samples of a trajectory, than information about correlations between these samples can be expressed in ${\bf P}_{{\bf n}_t}$. For example, the states (action-states) dependent kernel suggested by \citet{feng2019kernel} can be considered as ${\bf P}_{{\bf n}_t}$ in our framework. Furthermore,  ${\bf P}_{{\bf n}_t}$ can incorporate importance sampling ratio in off-policy learning \cite{mahmood2014weighted}. We leave the investigation of these interesting avenues to future work. 

\section{Experiments}
\label{Section:experiments}

We now present experiments that illustrate the performance attained by KOVA \footnote{Code is available at https://github.com/sdicastro/KOVA.}. Technical details on policy and VBF networks, on hyper-parameters grid search, on the hyper-parameters we used and on the running time of the algorithms are described in the supplementary material.

{\bf Maze environment}: We demonstrate KOVA performance in a policy evaluation task. An episode begins when an agent is positioned at the top-left corner (start point) of a $(10\!\times\!10)$-maze. The agent receives $-0.04$ reward for arriving to a new cell in the maze, while it receives $-0.25$ reward for returning to pre-visited cells. A success is defined if the agent arrives to the exit at the bottom-right corner. A loss is defined when the total reward is below $-50$ and the agent haven't arrived yet to the exit. In both cases the episode ends, the agent is positioned at the start point and begins a new episode. States are represented as images of the maze and the VBF (Q-function) is approximated with a DNN. The agent uses an $\epsilon$-greedy policy for selecting actions: 'top', 'down', 'right' or 'left. We tested double Q-learning \cite{van2016deep}, an off-policy algorithm for learning the Q-function parameters. We compared KOVA vs. Adam \cite{kingma2014adam} for policy evaluation. KOVA minimizes $L_t^{\text{EKF}}$ while Adam minimizes $L_t^{\text{MLE}}$ in each optimization step. For KOVA, we set ${\bf P}_{{\bf n}_t} = N {\bf I}$ and tried different evolution noises, calculated as ${\bf P}_{{\bf v}_t}\!=\!\frac{\eta}{1\!-\!\eta} {\bf P}_{t-1|t-1}$ with different values of $\eta$. As discussed in Section \ref{Sec:uncertainties}, it corresponds to a decay factor for the weight of previous observations. This is similar to the decaying learning rate in Adam optimizer, therefore we tested different constant and decaying learning rates for Adam. In figure \ref{fig:maze_exp} we present the agent success rate during training calculated over last 50 episodes. We can see that KOVA outperforms Adam in this task. The different values of $\eta$ affected the time it took the agent to complete the task. 

\begin{figure}[t]
	\centering{
		\includegraphics[width=.9\linewidth,keepaspectratio]{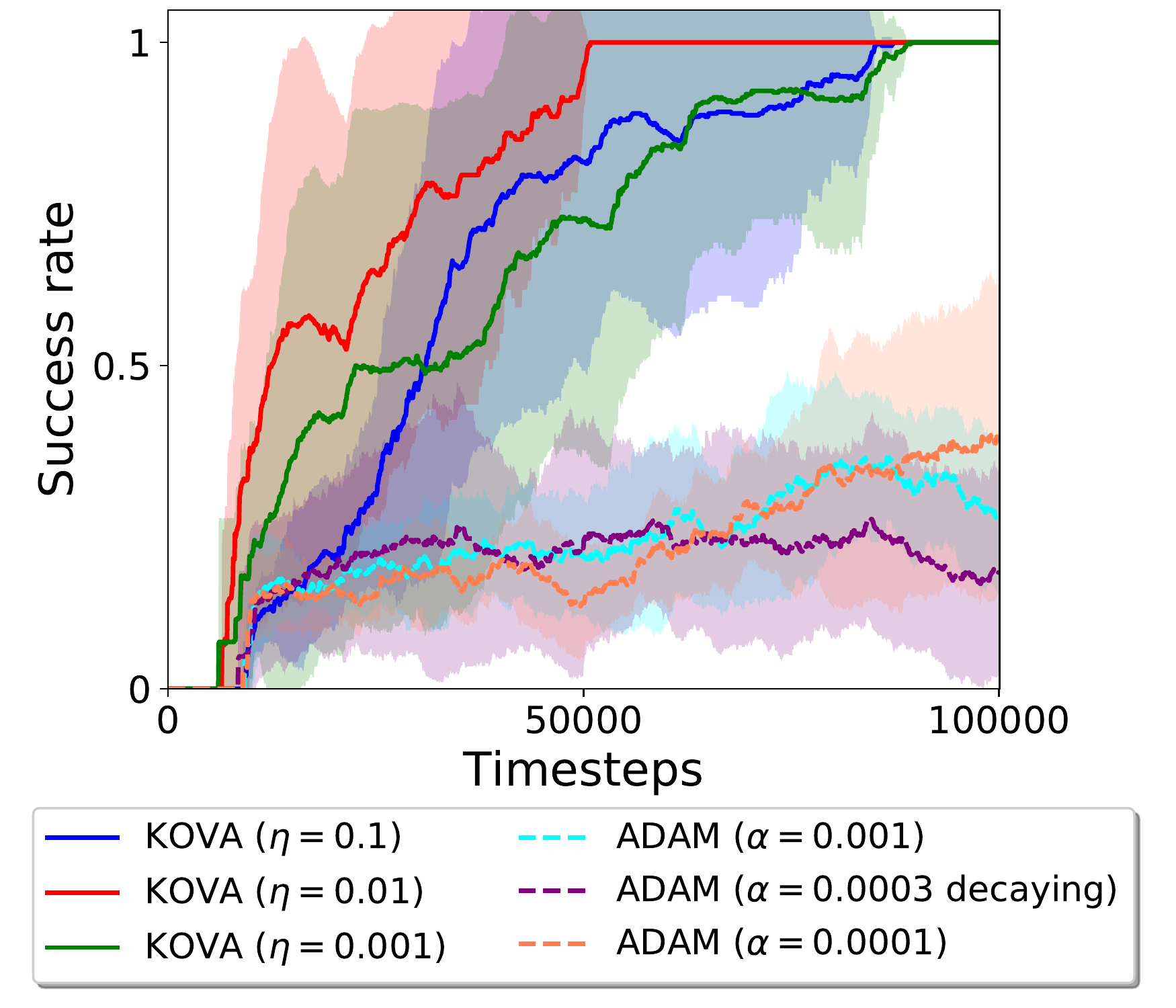}} 
	\caption{Success rate during training over $10\!\times\!10$ maze. We present the average (lines) and standard deviation (shaded area) of the success rate over 8 runnings, generated from random seeds.} 
	\label{fig:maze_exp}
\end{figure}

\begin{figure*}[t]
	\centering{
		\includegraphics[width=1.0\linewidth,height=0.6\textheight,keepaspectratio]{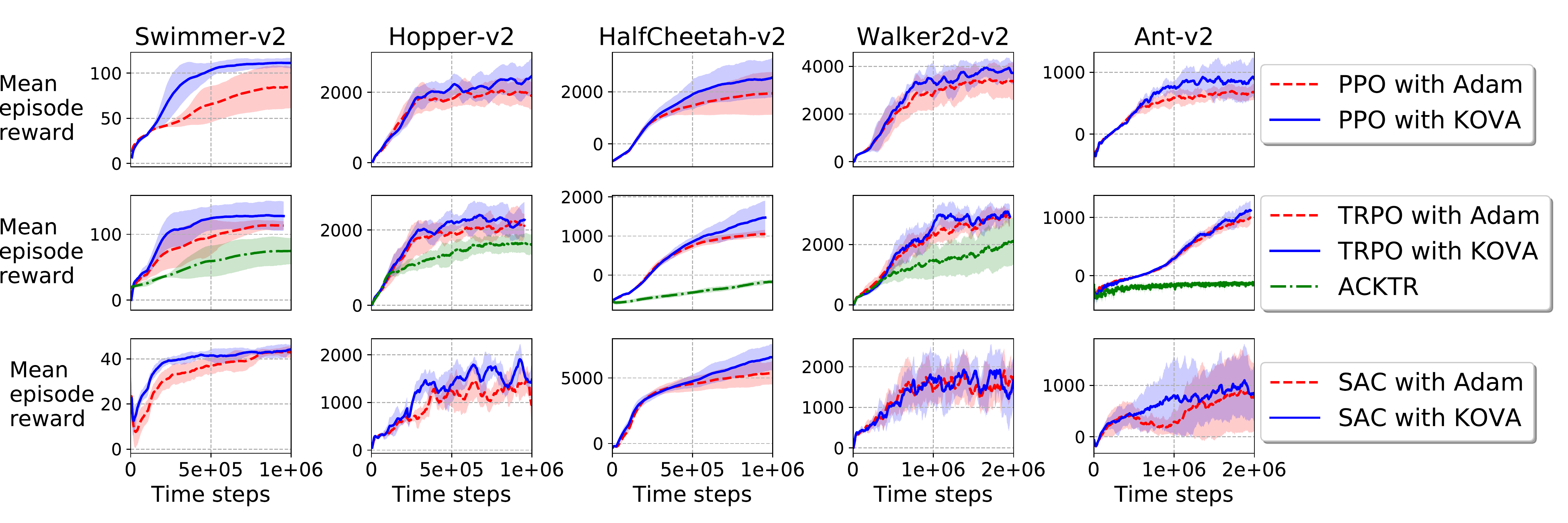}} 
	\caption{Mean episode reward during training for Mujoco environments. We compare KOVA vs. Adam for policy evaluation in different policy optimization algorithms: PPO ({\bf top row}), TRPO ({\bf middle row}), and SAC ({\bf bottom row}). ACKTR in presented as well in the middle row. We present the average (solid lines) and standard deviation (shaded area) of the mean episode reward over 8 runnings, generated from random seeds.} 
	\label{fig:mujoco_reward}
\end{figure*} 

{\bf Mujoco environment:} We test the performance of KOVA in domains with continuous state and action spaces: the robotic tasks benchmarks implemented in OpenAI Gym \cite{brockman2016openai}, which use the MuJoCo physics engine \cite{todorov2012mujoco}. In this experiment we test KOVA which is incorporated as the policy evaluation component in both on-policy and off-policy algorithms for policy training: PPO \cite{schulman2017proximal} and TRPO \cite{schulman2015trust} with their baselines implementations \cite{baselines}; SAC \cite{haarnoja2018soft} with stable baselines implementation \cite{stable-baselines}. For VBF training we replace the originally used Adam optimizer with our KOVA optimizer (Algoritm \ref{alg:KOVA}) and compare their effect on the mean episode reward in each environment. In addition, we test the performance of ACKTR \cite{wu2017scalable} which uses second order optimization for both policy and value functions. Unfortunately, comparison to KTD and GPTD is not possible since they are not applicable in these domains, as we explained in Section \ref{Section:KTD}. For KOVA with PPO and TRPO in the on-policy setting, we set a diagonal ${\bf P}_{{\bf n}_t}$ with $\sigma_i=N \max (1, \frac{1}{\frac{\pi_{\text{old}} (a_i|s_i)}{\pi_{\text{new}}(a_i|s_i)} + \epsilon})$. For KOVA with SAC in the off-policy setting we set ${\bf P}_{{\bf n}_t} = N {\bf I}$. We used the same update rule for ${\bf P}_{{\bf v}_t}$ as in the maze environment. Additional details can be found in the supplementary material. 

During training, each episode reward was recorded and in every optimization update we calculated the mean over the last 100 episodes. An optimization update is executed after several timesteps, defined by the horizon hyper-parameter of the algorithm. The results are presented in Figure \ref{fig:mujoco_reward}, where we chose to use timesteps in order to compare between different algorithms on the same graph. We can see that KOVA improves the agent's performance on most of the environments, while on others, it keeps approximately the same performance. On all environments, KOVA outperforms ACKTR. 

\section{Related Work}
\label{Section:related work}

{\bf Kalman filters:} The use of Kalman filters as optimizers is discussed by \citet{haykin2001kalman,vuckovic2018kalman,gomez2018decoupled}. In the RL context \citet{wilson2009neural} solved the dynamics of each parameter with Kalman filtering. \citet{wang2018batch} used Kalman filter for normalizing batches. \citet{choi2006generalized} used a kalman filter for fixed point approximation of the value function, however only linear parametrization was considered. Our work extends GPTD \cite{engel2003bayes,engel2005reinforcement,ghavamzadeh2016bayesian} and KTD \cite{geist2010kalman} for deep RL tasks. We discussed and compared KOVA to these methods through the paper. We have found that the improvement we made were crucial in order to enjoy the Kalman perspective in modern RL tasks. 

{\bf Bayesian Neural Networks (BNNs):} Bayesian methods place uncertainty on the approximator parameters \cite{blundell2015weight,gal2016dropout,khan2018fast}. \citet{depeweg2016learning,depeweg2017decomposition} have used BNNs for learning MDP dynamics in RL tasks. In these works a fully factorized Gaussian distribution on parameters is assumed  while we consider possible correlations between parameters. In addition, BNNs require sampling the parameters, and running several feed-forward runs for each sample. Our method avoids multiple samples of the parameters, since the uncertainty is propagated with every optimization update. 

{\bf Trust region and second order optimization methods:} EKF is connected with the incremental Gauss-Newton method \cite{bertsekas1996incremental} and with the on-line natural gradient \cite{amari1998natural,ollivier2018online}. Natural gradient in RL is mostly used in policy gradient algorithms to estimate policy parameters \cite{kakade2002natural,peters2008natural,schulman2015trust}, forming a trust region over these parameters \cite{schulman2015trust,schulman2017proximal}. For policy evaluation,  \citet{givchi2015quasi} and \citet{pan2017accelerated} proposed a quasi-newton method to approximate value functions with linear parametrization. Unfortunately, second order methods for non-linear parametrized value-based functions are rarely presented in the RL literature. \citet{wu2017scalable} suggested to apply the natural gradient method also on the critic in the actor-critic framework, using Kronecker-factored approximations (ACKTR). \citet{schulman2015high} suggested to apply Gauss-Newton method to estimate the VBF.  

Comparing our work with \citet{wu2017scalable} and \citet{schulman2015high}, our work is novel for two reasons: first, their works lack the analysis and formulation of the underlying model and assumptions that lead to the regularization in the objective function, while this is the focus of our work; second, KOVA avoids inverting ${\bf P}_{t|t-1}$ and hence is more computationally efficient in its update procedure, as discussed in Section \ref{section:KOVA}.

{\bf Distributional perspective on values and observations:} Distributional RL \cite{bellemare2017distributional} treats general distributions of total return, and considers VBF parameters as deterministic. In our work we assume Gaussian distribution over the total return and over the VBF parameters. The line of research on Bellman uncertainties \cite{o2017uncertainty} and on using randomized priors in VBFs estimation process \cite{osband2014generalization,osband2017deep,osband2018randomized,janz2019successor} share some similar concepts as ours, mainly on treating the VBF as an RV and on the difference between evolution noise and observation noise. However, their methods rely on excessive sampling from different distributions while our algorithm does not involve sampling. KOVA may be extended by sampling from the covariances it learns, although this remains for future work.

\section{Discussion \& Conclusion}
In this work we presented a Kalman-based framework for policy evaluation in RL and proposed a regularized objective function for optimizing value based functions. We have succeeded in using a Kalman-based approach for evaluating policies in deep RL tasks, circumventing the limitations of KTD and GPTD. The incorporation of KOVA as a policy evaluation component in popular policy optimization algorithms is straightforward. Our empirical results illustrate how KOVA improves the performance of various agents in deep RL tasks, both in on-policy and off-policy settings. These improvements demonstrate the importance of incorporating uncertainty estimation in value function approximation and suggest that EKF should not be neglected and should be considered as a better optimizer for VBFs.
We believe that the improvements come from properly designing noise covariances ${\bf P}_{{\bf v}_t}$ and ${\bf P}_{{\bf n}_t}$. Further investigation of their roles and possibly learning these covariances, may further improve the performance of KOVA. For future work, it would be interesting to investigate the connection between KOVA and complex prior methods that use prior knowledge to encourage exploration. The relation between trust-region over value parameters and trust-region over policy parameters may be an additional interesting avenue to improve exploration. 

{\small 
\bibliography{KalmanOptimizationforValueApproximation}
\bibliographystyle{icml2020}}

\onecolumn
\appendix
\numberwithin{equation}{section}
\counterwithin{table}{section}
\section*{\LARGE Supplementary Material}
\section{Theoretical Results}
\subsection{Extended Kalman Filter (EKF)}
\label{suppSection:EKF}
In this section we briefly outline the Extended Kalman filter \cite{anderson1979optimal,gelb1974applied}. The EKF considers the following model:
\begin{equation}
\label{suppeq:Extended-Kalman}
\begin{cases}
\boldsymbol{\theta}_t = \boldsymbol{\theta}_{t-1} + {\bf v}_t\\
y ({\bf u}_t) = h({\bf u}_t; \boldsymbol{\theta}_t) + {\bf n}_t
\end{cases},
\end{equation}
where $\boldsymbol{\theta}_t \in \mathbb{R}^{d \times 1}$ is a parameter evaluated at time $t$, $y ({\bf u}_t) = [ 
y(u_t^1),  y(u_t^2), \ldots, y(u_t^N)  ]^\top \in \mathbb{R}^{N \times 1}$ is the $N$-dimensional observation vector at time $t$ for some {\it input vector} ${\bf u}_t \in \mathbb{R}^{N \times 1}$, and $h({\bf u}_t; \boldsymbol{\theta}_t) = [
h(u_t^1; \boldsymbol{\theta}_t),
h(u_t^2; \boldsymbol{\theta}_t), 
\ldots,
h(u_t^N; \boldsymbol{\theta}_t)  
^\top \in \mathbb{R}^{N \times 1}$
where $h(u; \boldsymbol{\theta})$ is a non-linear observation function with input $u$ and parameters $\boldsymbol{\theta}$. The evolution noise ${\bf v}_t$ is white ($\mathbb{E}[{\bf v}_t] = {\bf 0}$) with covariance ${\bf P}_{{\bf v}_t} \triangleq \mathbb{E}[{\bf v}_t {\bf v}_t^\top]$, $\ \mathbb{E}[{\bf v}_t {\bf v}_{t'}^\top] = {\bf 0}, \quad \forall t \neq t'$. The observation noise ${\bf n}_t$ is white ($\mathbb{E}[{\bf n}_{t}]={\bf 0}$) with covariance ${\bf P}_{{\bf n}_t} \triangleq \mathbb{E}[{\bf n}_t {\bf n}_t^\top]$, $\ \mathbb{E}[{\bf n}_t {\bf n}_{t'}^\top] = {\bf 0}, \quad \forall t \neq t'$.

The EKF sets the estimation of the parameters $\boldsymbol{\theta}$ at time $t$ according to the conditional expectation: 
\begin{align}
\label{suppeq:weight_estimation}
\nonumber \boldsymbol{\hat{\theta}}_{t|t} & \triangleq \mathbb{E} [ \boldsymbol{\theta}_t | y_{1:t}] \\ 
\boldsymbol{\hat{\theta}}_{t|t-1} & \triangleq \mathbb{E} [ \boldsymbol{\theta}_t | y_{1:t-1}]  = \boldsymbol{\hat{\theta}}_{t-1|t-1} 
\end{align}

where with some abuse of notation, $y_{1:t'}$ are the observations gathered up to time $t'$: $y ({\bf u}_1), \ldots, y ({\bf u}_t')$. The {\it parameters errors} are defined by: 
\[\boldsymbol{\tilde{\theta}}_{t|t}  \triangleq \boldsymbol{\theta}_t - \boldsymbol{\hat{\theta}}_{t|t} \]
\begin{equation}
\label{suppeq:weights_error}
\boldsymbol{\tilde{\theta}}_{t|t-1} \triangleq \boldsymbol{\theta}_t - \boldsymbol{\hat{\theta}}_{t|t-1}
\end{equation}

The conditional {\it error covariances} are given by: 
\begin{align*}
{\bf P}_{t|t}  & \triangleq \mathbb{E} \big[  \boldsymbol{\tilde{\theta}}_{t|t}  \boldsymbol{\tilde{\theta}}_{t|t}^\top | y_{1:t}\big], \\
{\bf P}_{t|t-1} & \triangleq \mathbb{E} \big[  \boldsymbol{\tilde{\theta}}_{t|t-1}  \boldsymbol{\tilde{\theta}}_{t|t-1}^\top | y_{1:t-1}\big] \\
& = \mathbb{E} \big[  (\boldsymbol{\theta}_t - \boldsymbol{\hat{\theta}}_{t|t-1}) (\boldsymbol{\theta}_t - \boldsymbol{\hat{\theta}}_{t|t-1})^\top | y_{1:t-1}\big]\\
& = \mathbb{E} \big[  (\boldsymbol{\theta}_{t-1} + {\bf v}_t - \boldsymbol{\hat{\theta}}_{t-1|t-1}) (\boldsymbol{\theta}_{t-1} + {\bf v}_t - \boldsymbol{\hat{\theta}}_{t-1|t-1})^\top | y_{1:t-1}\big] \\
&  = \mathbb{E} \big[  (\boldsymbol{\tilde{\theta}}_{t-1|t-1} + {\bf v}_t) (\boldsymbol{\tilde{\theta}}_{t-1|t-1} + {\bf v}_t)^\top | y_{1:t-1}\big]\\
& = \mathbb{E} \big[  (\boldsymbol{\tilde{\theta}}_{t-1|t-1} \boldsymbol{\tilde{\theta}}_{t-1|t-1}^\top | y_{1:t-1}\big] + 2 \mathbb{E} \big[ \boldsymbol{\tilde{\theta}}_{t-1|t-1} {\bf v}_t^\top| y_{1:t-1} \big]  + \mathbb{E} \big[ {\bf v}_t {\bf v}_t^\top | y_{1:t-1}\big]\\
& = {\bf P}_{t-1|t-1} + {\bf P}_{{\bf v}_t}.
\end{align*}
\begin{equation}
\label{suppeq:error_covariance}
\boxed{{\bf P}_{t|t-1} = {\bf P}_{t-1|t-1} + {\bf P}_{{\bf v}_t}}
\end{equation}
\\
\\

EKF considers several statistics of interest at each time step.

{\it The prediction of the observation function}:
\[{\bf \hat{y}}_{t|t-1} \triangleq  \mathbb{E}[h({\bf u}_t; \boldsymbol{\theta}_t)|y_{1:t-1}].\]
{\it The observation innovation}:
\[{\bf \tilde{y}}_{t|t-1} \triangleq h({\bf u}_t; \boldsymbol{\theta}_t) - {\bf \hat{y}}_{t|t-1}.\]
{\it The covariance between the parameters error and the innovation}:
\[{\bf P}_{\boldsymbol{\tilde{\theta}}_t,{\bf \tilde{y}}_{t}}  \triangleq \mathbb{E}[ \boldsymbol{\tilde{\theta}}_{t|t-1} {\bf \tilde{y}}_{t|t-1}^\top |y_{1:t-1}].\]
{\it The covariance of the innovation}:
\[ {\bf P}_{{\bf \tilde{y}}_t}  \triangleq \mathbb{E}[( {\bf \tilde{y}}_{t|t-1} {\bf \tilde{y}}_{t|t-1}^\top|y_{1:t-1}] + {\bf P}_{{\bf n}_t}. \]
The {\it Kalman gain}:
\[{\bf K}_t \triangleq {\bf P}_{\boldsymbol{\tilde{\theta}}_t,{\bf \tilde{y}}_{t}}  {\bf P}_{{\bf \tilde{y}}_t}^{-1}.\]
The above statistics serve for the update of the parameters and the error covariance:
\begin{equation}
\label{suppeq:kalman update}
\begin{cases}
\boldsymbol{\hat{\theta}}_{t|t} = \boldsymbol{\hat{\theta}}_{t|t-1} +    {\bf K}_t \big( y({\bf u}_t) - h({\bf u}_t; \boldsymbol{\hat{\theta}}_{t|t-1}) \big),\\
{\bf P}_{t|t} = {\bf P}_{t|t-1} - {\bf K}_t  {\bf P}_{{\bf \tilde{y}}_t} {\bf K}_t^\top.
\end{cases}
\end{equation}

\subsection{EKF for Value-Based Function Estimation}
\label{suppSection:EKF_Qlearning}
When applying the EKF formulation to value-based function (VBF) approximation, the observation at time $t$ is  the target label $y({\bf u}_t)$ (see Table 1 in the main article), and the observation function $h$ can be the state value function or the state-action value function.   

The EKF uses a first order Taylor series linearization for the observation function:
\begin{equation}
\label{suppeq:Linearization}
h({\bf u}_t; \boldsymbol{\theta}_t) 
= h({\bf u}_t; \boldsymbol{\hat{\theta}}) +   \nabla_{\boldsymbol{\theta}_t} h({\bf u}_t; \boldsymbol{\hat{\theta}})^\top \big( \boldsymbol{\theta}_{t} - \boldsymbol{\hat{\theta}} \big),
\end{equation}
where $\nabla_{\boldsymbol{\theta}_t} h({\bf u}_t; \boldsymbol{\hat{\theta}})  = \begin{bmatrix} 
\nabla_{\boldsymbol{\theta}_t} h(u_t^1; \boldsymbol{\hat{\theta}}) , \ldots , \nabla_{\boldsymbol{\theta}_t} h(u_t^N; \boldsymbol{\hat{\theta}})  \end{bmatrix} \in \mathbb{R}^{d \times N}$
and $\boldsymbol{\hat{\theta}}$ is typically chosen to be the previous estimation of the parameters at time $t-1$,  $\boldsymbol{\hat{\theta}}_{t|t-1}$. This linearization helps in computing the statistics of interest. Recall that the expectation here is over the random variable $\boldsymbol{\theta}_t$ where $\boldsymbol{\hat{\theta}}_{t|t-1}$ is fixed. For simplicity, we keep to write $\boldsymbol{\hat{\theta}}$. 
{\it The prediction of the observation function} is:
\begin{align*}
{\bf \hat{y}}_{t|t-1} & \triangleq  \mathbb{E}[h({\bf u}_t; \boldsymbol{\theta}_t)|y_{1:t-1}]\\
& \underbrace{=}_{(\ref{suppeq:Linearization})} \mathbb{E}\Big[h({\bf u}_t,\boldsymbol{\hat{\theta}})
+   \nabla_{\boldsymbol{\theta}_t} h({\bf u}_t; \boldsymbol{\hat{\theta}})^\top \big( \boldsymbol{\theta}_{t} - \boldsymbol{\hat{\theta}} \big) |y_{1:t-1}\Big]\\
& = h({\bf u}_t; \boldsymbol{\hat{\theta}})  +   \nabla_{\boldsymbol{\theta}_t} h({\bf u}_t; \boldsymbol{\hat{\theta}})^\top \big( \mathbb{E}[  \boldsymbol{\theta}_{t} |y_{1:t-1}] - \boldsymbol{\hat{\theta}} \big)\\
& \underbrace{=}_{(\ref{suppeq:weight_estimation})} h({\bf u}_t; \boldsymbol{\hat{\theta}})  +    \nabla_{\boldsymbol{\theta}_t} h({\bf u}_t; \boldsymbol{\hat{\theta}})^\top \big( \boldsymbol{\hat{\theta}} - \boldsymbol{\hat{\theta}} \big)\\
& = h({\bf u}_t; \boldsymbol{\hat{\theta}}) = h({\bf u}_t; \boldsymbol{\hat{\theta}}_{t|t-1})
\end{align*}
We conclude that:
\begin{equation}
\label{suppeq:prediction_observation}
\boxed{{\bf \hat{y}}_{t|t-1}   = h({\bf u}_t; \boldsymbol{\hat{\theta}}_{t|t-1}) }
\end{equation}
{\it The observation innovation}:
\begin{equation}
\label{suppeq:observation_innovation}
\boxed{ {\bf \tilde{y}}_{t|t-1} \triangleq h({\bf u}_t; \boldsymbol{\theta}_t) -{\bf \hat{y}}_{t|t-1}  \underbrace{=}_{(\ref{suppeq:prediction_observation})} h({\bf u}_t,  \boldsymbol{\theta}_t) - h({\bf u}_t; \boldsymbol{\hat{\theta}}_{t|t-1}) }
\end{equation} 
Let's simplify the following:
\begin{align}
\label{suppeq:simplify}
\nonumber  h({\bf u}_t,  \boldsymbol{\theta}_t) - h({\bf u}_t; \boldsymbol{\hat{\theta}}) & \underbrace{=}_{(\ref{suppeq:Linearization})} \big( \cancel{h({\bf u}_t; \boldsymbol{\hat{\theta}})}  +     \nabla_{\boldsymbol{\theta}_t} h({\bf u}_t; \boldsymbol{\hat{\theta}})^\top \big( \boldsymbol{\theta}_{t} - \boldsymbol{\hat{\theta}} \big) - \cancel{h({\bf u}_t; \boldsymbol{\hat{\theta}}) \big)}\\
& = \nabla_{\boldsymbol{\theta}_t} h({\bf u}_t; \boldsymbol{\hat{\theta}})^\top \big( \boldsymbol{\theta}_{t} - \boldsymbol{\hat{\theta}} \big)
\end{align}
{\it The covariance between the parameters error and the innovation} (here we also denote $\boldsymbol{\hat{\theta}} = \boldsymbol{\hat{\theta}}_{t|t-1}$):
\begin{align*}
{\bf P}_{\boldsymbol{\tilde{\theta}}_t,{\bf \tilde{y}}_{t}}  & \triangleq \mathbb{E}[ \boldsymbol{\tilde{\theta}}_{t|t-1} {\bf \tilde{y}}_{t|t-1}^\top |y_{1:t-1}]\\
& \underbrace{=}_{(\ref{suppeq:weights_error}) + (\ref{suppeq:observation_innovation})} \mathbb{E}[ \big(\boldsymbol{\theta}_t - \boldsymbol{\hat{\theta}} \big) \big( h({\bf u}_t; \boldsymbol{\theta}_t) - h({\bf u}_t; \boldsymbol{\hat{\theta}}) \big)^\top |y_{1:t-1}]\\
& \underbrace{=}_{(\ref{suppeq:simplify})} \mathbb{E}[ \big(\boldsymbol{\theta}_t - \boldsymbol{\hat{\theta}} \big)  \big( \boldsymbol{\theta}_{t} - \boldsymbol{\hat{\theta}} \big)^\top \nabla_{\boldsymbol{\theta}_t} h({\bf u}_t; \boldsymbol{\hat{\theta}})  |y_{1:t-1}]\\
& \underbrace{=}_{(\ref{suppeq:weights_error})} \mathbb{E}[ \boldsymbol{\tilde{\theta}}_{t|t-1}  \boldsymbol{\tilde{\theta}}_{t|t-1}^\top  |y_{1:t-1}] \nabla_{\boldsymbol{\theta}_t} h({\bf u}_t; \boldsymbol{\hat{\theta}})\\
& \underbrace{=}_{(\ref{suppeq:error_covariance})} {\bf P}_{t|t-1} \nabla_{\boldsymbol{\theta}_t} h({\bf u}_t; \boldsymbol{\hat{\theta}}_{t|t-1})
\end{align*}
\begin{equation}
\label{suppeq:covariance_weights_innovation}
\boxed{{\bf P}_{\boldsymbol{\tilde{\theta}}_t,{\bf \tilde{y}}_{t}} = {\bf P}_{t|t-1} \nabla_{\boldsymbol{\theta}_t} h({\bf u}_t; \boldsymbol{\hat{\theta}}_{t|t-1})}
\end{equation}
{\it The covariance of the innovation}:
\begin{align*}
{\bf P}_{{\bf \tilde{y}}_t}  & \triangleq \mathbb{E}[( {\bf \tilde{y}}_{t|t-1} {\bf \tilde{y}}_{t|t-1}^\top|y_{1:t-1}] + {\bf P}_{{\bf n}_t}\\
& \underbrace{=}_{(\ref{suppeq:observation_innovation})} \mathbb{E}[ \big( h({\bf u}_t,  \boldsymbol{\theta}_t) - h({\bf u}_t; \boldsymbol{\hat{\theta}}) \big) \big( h({\bf u}_t,  \boldsymbol{\theta}_t) - h({\bf u}_t; \boldsymbol{\hat{\theta}}) \big)^\top |y_{1:t-1}] + {\bf P}_{{\bf n}_t}\\
& \underbrace{=}_{(\ref{suppeq:simplify})} \mathbb{E} \Big[     \nabla_{\boldsymbol{\theta}_t} h({\bf u}_t; \boldsymbol{\hat{\theta}})^\top \big( \boldsymbol{\theta}_{t} - \boldsymbol{\hat{\theta}} \big)  \big( \boldsymbol{\theta}_{t} - \boldsymbol{\hat{\theta}} \big)^\top   \nabla_{\boldsymbol{\theta}_t} h({\bf u}_t; \boldsymbol{\hat{\theta}})  |y_{1:t-1} \Big] + {\bf P}_{{\bf n}_t}\\
& \underbrace{=}_{(\ref{suppeq:weights_error})} \nabla_{\boldsymbol{\theta}_t} h({\bf u}_t,   \boldsymbol{\hat{\theta}})^\top \mathbb{E}[ \boldsymbol{\tilde{\theta}}_{t|t-1}  \boldsymbol{\tilde{\theta}}_{t|t-1}^\top    |y_{1:t-1}] \nabla_{\boldsymbol{\theta}_t} h({\bf u}_t,  \boldsymbol{\hat{\theta}}) + {\bf P}_{{\bf n}_t}\\
& =\nabla_{\boldsymbol{\theta}_t} h({\bf u}_t,   \boldsymbol{\hat{\theta}})^\top {\bf P}_{t|t-1} \nabla_{\boldsymbol{\theta}_t} h({\bf u}_t,   \boldsymbol{\hat{\theta}})^\top + {\bf P}_{{\bf n}_t}
\end{align*}
\begin{equation}
\label{suppeq:covariance_innovation}
\boxed{ {\bf P}_{{\bf \tilde{y}}_t}   = \nabla_{\boldsymbol{\theta}_t} h({\bf u}_t; \boldsymbol{\hat{\theta}}_{t|t-1})^\top {\bf P}_{t|t-1} \nabla_{\boldsymbol{\theta}_t} h({\bf u}_t; \boldsymbol{\hat{\theta}}_{t|t-1}) + {\bf P}_{{\bf n}_t}}
\end{equation}
The {\it Kalman gain}:
\begin{align}
\label{suppeq:kalman_gain}
\nonumber {\bf K}_t &   \triangleq {\bf P}_{\boldsymbol{\tilde{\theta}}_t,{\bf \tilde{y}}_{t}}  {\bf P}_{{\bf \tilde{y}}_t}^{-1}\\
& \underbrace{=}_{(\ref{suppeq:covariance_weights_innovation}) + (\ref{suppeq:covariance_innovation})} {\bf P}_{t|t-1} \nabla_{\boldsymbol{\theta}_t} h({\bf u}_t,  \boldsymbol{\hat{\theta}}_{t|t-1}) \Big( \nabla_{\boldsymbol{\theta}_t} h({\bf u}_t,  \boldsymbol{\hat{\theta}}_{t|t-1})^\top {\bf P}_{t|t-1} \nabla_{\boldsymbol{\theta}_t} h({\bf u}_t,  \boldsymbol{\hat{\theta}}_{t|t-1})  + {\bf P}_{{\bf n}_t} \Big)^{-1}
\end{align}

and the update for the parameters of the VBF and the error covariance are the same as in Equation (\ref{suppeq:kalman update}) as we prove in Theorem 1.

\newpage
\subsection{The MAP estimator}
We adopt the Bayesian approach in which we are interested in finding the optimal set of parameters $\boldsymbol{\theta}_t$ that maximizes the posterior distribution of the parameters given the observations we have gathered up to time $t$, denoted as the  $y_{1:t}$.

According to Bayes rule, the posterior distribution is defined as: 
\[p(\boldsymbol{\theta}_t|y_{1:t}) = \frac{p(y_{1:t}|\boldsymbol{\theta}_t) p(\boldsymbol{\theta}_t)}{p(y_{1:t})}\]
where $p(y_{1:t}|\boldsymbol{\theta})$ is the {\it likelihood} of the observations given the parameters $\boldsymbol{\theta}$ and $p(\boldsymbol{\theta})$ is the {\it prior} distribution over $\boldsymbol{\theta}$.	We will expend the term of the posterior \cite{van2004sigma}:
\begin{align}
\nonumber p(\boldsymbol{\theta}_t|y_{1:t}) & = \frac{p(y_{1:t}|\boldsymbol{\theta}_t)p(\boldsymbol{\theta}_t)}{p(y_{1:t})} \\
& = \frac{p(y_t|y_{1:t-1},\boldsymbol{\theta}_t) p(y_{1:t-1}|\boldsymbol{\theta}_t)p(\boldsymbol{\theta}_t)}{p(y_{1:t})} \label{eq:PosteriorWithNoise1} \\ 
& = \frac{p(y_t|\boldsymbol{\theta}_t) p(y_{1:t-1}|\boldsymbol{\theta}_t)p(\boldsymbol{\theta}_t)  }{p(y_{1:t})} \cdot \frac{p(y_{1:t-1})}{p(y_{1:t-1})}  \label{eq:PosteriorWithNoise2}\\ 
& = \frac{p(y_t|\boldsymbol{\theta}_t)p(\boldsymbol{\theta}_t|y_{1:t-1}) p(y_{1:t-1}) }{p(y_{1:t})} \label{eq:PosteriorWithNoise3}
\end{align}
The transition in (\ref{eq:PosteriorWithNoise1}) is according to the conditional probability:
\begin{align*}
p(y_{1:t}|\boldsymbol{\theta}_t) & = p(y_t, y_{1:t-1}|\boldsymbol{\theta}_t)\\
& = \frac{p(y_t, y_{1:t-1},\boldsymbol{\theta}_t)}{p(\boldsymbol{\theta}_t)} \\
& = \frac{p(y_{1:t-1},\boldsymbol{\theta}_t)p(y_t|y_{1:t-1},\boldsymbol{\theta}_t)}{p(\boldsymbol{\theta}_t)} \\
& = p(y_{1:t-1}|\boldsymbol{\theta}_t)p(y_t|y_{1:t-1},\boldsymbol{\theta}_t) 
\end{align*}
The transition in (\ref{eq:PosteriorWithNoise2}) is according to the conditional independence: $p(y_t|y_{1:t-1},\boldsymbol{\theta}_t) = p(y_t|\boldsymbol{\theta}_t)$, and we multiplied the numerator and the dominator by $p(y_{1:t-1})$.\\
The transition in (\ref{eq:PosteriorWithNoise3}) is according to Bayes rule: $p(\boldsymbol{\theta}_t|y_{1:t-1}) = \frac{p(y_{1:t-1}|\boldsymbol{\theta}_t)p(\boldsymbol{\theta}_t)}{p(y_{1:t-1})}$.

The MAP estimator for $\boldsymbol{\theta}_t$ is the one who maximizes the posterior distribution described in (\ref{eq:PosteriorWithNoise3}).
\begin{align}
\nonumber \boldsymbol{\theta}_{t}^{MAP}  = & \arg\max_{\boldsymbol{\theta}_t}  \big\{ p(\boldsymbol{\theta}_t|y_{1:t}) \big\}\\
\nonumber = &\arg\max_{\boldsymbol{\theta}_t} \Big\{  \frac{p(y_t|\boldsymbol{\theta}_t)p(\boldsymbol{\theta}_t|y_{1:t-1}) p(y_{1:t-1}) }{p(y_{1:t})} \Big\}\\
\nonumber = &\arg\max_{\boldsymbol{\theta}_t} \big\{  p(y_t|\boldsymbol{\theta}_t)p(\boldsymbol{\theta}_t|y_{1:t-1})  \big\}\\
\nonumber = &\arg\max_{\boldsymbol{\theta}_t} \big\{ \log \big(  p(y_t|\boldsymbol{\theta}_t)p(\boldsymbol{\theta}_t|y_{1:t-1})  \big) \big\} \\
\nonumber = &\arg\max_{\boldsymbol{\theta}_t} \big\{ \log   p(y_t|\boldsymbol{\theta}_t) + \log p(\boldsymbol{\theta}_t|y_{1:t-1})   \big\} \\
= &\arg\min_{\boldsymbol{\theta}_t} \big\{ - \log   p(y_t|\boldsymbol{\theta}_t) - \log p(\boldsymbol{\theta}_t|y_{1:t-1})   \big\} \label{eq:MAPln}
\end{align}
In (\ref{eq:MAPln}) We used the derivation in ($\ref{eq:PosteriorWithNoise3}$) and the fact that the argument which maximizes the posterior is the same as the argument that maximizes the $\log (\cdot)$ of the posterior. In addition this argument also minimizes the negative $\log(\cdot)$. 

We will replace here $y_t = y({\bf u}_t)$ and receive:
\begin{equation}
\label{eq:MAPln2}
\boldsymbol{\theta}_{t}^{MAP} = \arg\min_{\boldsymbol{\theta}_t}  \big\{ - \log   p(y({\bf u}_t)|\boldsymbol{\theta}_t) - \log p(\boldsymbol{\theta}_t|y_{1:t-1})   \big\}
\end{equation}

In order to solve (\ref{eq:MAPln2}), we consider the EKF formulation for the VBF parameters.

\subsection{Gaussian assumptions}
When estimating using the EKF, it is common to make the following assumptions regarding the likelihood and the posterior in Equation (\ref{eq:MAPln2}):
\begin{assumption}
	\label{suppAs:ConditionalIndependance}
	The likelihood $p(y({\bf u}_t)|\boldsymbol{\theta}_t)$ is assumed to be Gaussian: 
	$y({\bf u}_t)|\boldsymbol{\theta}_t \sim \mathcal{N}(h({\bf u}_t,  \boldsymbol{\theta}_t), {\bf P}_{{\bf n}_t})$. 
\end{assumption}

\begin{assumption}
	\label{suppAs:GaussianPosterior}
	The posterior distribution $p(\boldsymbol{\theta}_t|y_{1:t-1})$ is assumed to be Gaussian: $\boldsymbol{\theta}_t|y_{1:t-1} \sim \mathcal{N}(\boldsymbol{\hat{\theta}}_{t|t-1},{\bf P}_{t|t-1})$.
\end{assumption}
Following are the calculations for the means and covariances in Assumptions \ref{suppAs:ConditionalIndependance} and \ref{suppAs:GaussianPosterior}. For the likelihood $p(y({\bf u}_t)|\boldsymbol{\theta}_t)$:
\begin{align}
\label{suppeq:expected_y}
\mathbb{E} \big[y({\bf u}_t)|\boldsymbol{\theta}_t \big] & \underbrace{=}_{(\ref{suppeq:Extended-Kalman})} \mathbb{E} \big[h({\bf u}_t; \boldsymbol{\theta}_t) + {\bf n}_t|\boldsymbol{\theta}_t \big]  = \mathbb{E} \big[h({\bf u}_t; \boldsymbol{\theta}_t) |\boldsymbol{\theta}_t \big] + \underbrace{\mathbb{E} \big[ {\bf n}_t |\boldsymbol{\theta}_t\big]}_{={\bf 0}}  = h({\bf u}_t; \boldsymbol{\theta}_t)
\end{align}
Let's evaluate the following:
\begin{align}
\label{suppeq:diff_y}
y({\bf u}_t) - \mathbb{E} \big[y({\bf u}_t)|\boldsymbol{\theta}_t \big]  & \underbrace{=}_{(\ref{suppeq:Extended-Kalman}) + (\ref{suppeq:expected_y})} h({\bf u}_t; \boldsymbol{\theta}_t) + {\bf n}_t - h({\bf u}_t; \boldsymbol{\theta}_t) = {\bf n}_t
\end{align}
\begin{align*}
& Cov(y({\bf u}_t)|\boldsymbol{\theta}_t)  \triangleq 
\mathbb{E} \big[\big(y({\bf u}_t) - \mathbb{E} \big[y({\bf u}_t)|\boldsymbol{\theta}_t \big] \big)  \big(y({\bf u}_t) - \mathbb{E} \big[y({\bf u}_t)|\boldsymbol{\theta}_t \big] \big)^\top|\boldsymbol{\theta}_t \big]  \underbrace{=}_{(\ref{suppeq:diff_y})} \mathbb{E} \big[  {\bf n}_t {\bf n}_t^\top  |\boldsymbol{\theta}_t \big] = {\bf P}_{{\bf n}_t}
\end{align*}
For the posterior $p(\boldsymbol{\theta}_t|y_{1:t-1})$:  
$\mathbb{E}_{\boldsymbol{\theta}_t} \big[ \boldsymbol{\theta}_t|y_{1:t-1}\big] \underbrace{=}_{(\ref{suppeq:weight_estimation})} \boldsymbol{\hat{\theta}}_{t|t-1}$.
\begin{align*}
Cov(\boldsymbol{\theta}_t|y_{1:t-1}) \triangleq  
& \mathbb{E}_{\boldsymbol{\theta}_t} \big[ \big(\boldsymbol{\theta}_t - \boldsymbol{\hat{\theta}}_{t|t-1} \big) \big(\boldsymbol{\theta}_t - \boldsymbol{\hat{\theta}}_{t|t-1} \big)^\top |y_{1:t-1} \big] = \mathbb{E}_{\boldsymbol{\theta}_t} \big[ \boldsymbol{\tilde{\theta}}_{t|t-1} \boldsymbol{\tilde{\theta}}_{t|t-1}^\top |y_{1:t-1} \big] \underbrace{=}_{(\ref{suppeq:error_covariance})} {\bf P}_{t|t-1}
\end{align*}

\subsection{Proof of Theorem 1}
Based on the Gaussian assumptions, we can derive the following Theorem:
\begin{theorem}
	\label{supp:theorem1}
	Under Assumptions \ref{suppAs:ConditionalIndependance} and \ref{suppAs:GaussianPosterior}, $\boldsymbol{\hat{\theta}}^{\text{EKF}}_{t|t}$  (\ref{suppeq:kalman update}) minimizes at each time step $t$ the following regularized objective function:
	\begin{align}
	\label{suppeq:objective_EKF}
	L^{\text{EKF}}_t(\boldsymbol{\theta}_t) & =  \frac{1}{2}  \big(\delta ({\bf u}_t; \boldsymbol{\theta}_t) \big)^\top {\bf P}_{{\bf n}_t}^{-1} \big(\delta ({\bf u}_t; \boldsymbol{\theta}_t) \big) +  \frac{1}{2}(\boldsymbol{\theta}_t - \boldsymbol{\hat{\theta}}_{t|t-1})^\top {\bf P}_{t|t-1}^{-1} (\boldsymbol{\theta}_t - \boldsymbol{\hat{\theta}}_{t|t-1}).
	\end{align}
\end{theorem}

\begin{proof}
	We solve the minimization problem in (\ref{eq:MAPln2}) by substituting the Gaussian Assumptions \ref{suppAs:ConditionalIndependance} and \ref{suppAs:GaussianPosterior}. We show that this minimization problem is equivalent to minimize the objective function $L_t^{\text{EKF}}$ in Theorem \ref{supp:theorem1}.
	\begin{align*}
	& \nonumber  \boldsymbol{\hat{\theta}}_{t|t}^{\text{MAP}}  \underbrace{=}_{(\ref{eq:MAPln2})} \arg\min_{\boldsymbol{\theta}_t}  \big\{  - \log \Big( p(y({\bf u}_t)|\boldsymbol{\theta}_t) \Big) - \log \Big( p(\boldsymbol{\theta}_t|y_{1:t-1}) \Big) \big\}\\
	\nonumber & = \arg\min_{\boldsymbol{\theta}_t}  \Big\{ -\log \bigg( \frac{1}{(2\pi)^{N/2} |{\bf P}_{{\bf n}_t}|^{1/2} }  \exp \Big( - \frac{1}{2} \big( y({\bf u}_t) - h({\bf u}_t; \boldsymbol{\theta}_t) \big)^\top  {\bf P}_{{\bf n}_t}^{-1} \big( y({\bf u}_t) - h({\bf u}_t; \boldsymbol{\theta}_t) \big) \Big) \bigg) \\
	& -\log \bigg( \frac{1}{(2 \pi)^{d/2} |{\bf P}_{t|t-1}|^{1/2}} \exp \Big( - \frac{1}{2} (\boldsymbol{\theta}_t - \boldsymbol{\hat{\theta}}_{t|t-1})^\top  {\bf P}_{t|t-1}^{-1} (\boldsymbol{\theta}_t -  \boldsymbol{\hat{\theta}}_{t|t-1})  \Big) \bigg)  \\
	\nonumber & = \arg\min_{\boldsymbol{\theta}_i}  \Big\{ \frac{1}{2}  \big( y({\bf u}_t) - h({\bf u}_t; \boldsymbol{\theta}_t) \big)^\top  {\bf P}_{{\bf n}_t}^{-1} \big( y({\bf u}_t) - h({\bf u}_t; \boldsymbol{\theta}_t) \big)  -\log \Big( \frac{1}{(2\pi)^{N/2} |{\bf P}_{{\bf n}_t}|^{1/2} } \Big)  \\
	& + \frac{1}{2} (\boldsymbol{\theta}_t - \boldsymbol{\hat{\theta}}_{t|t-1})^\top  {\bf P}_{t|t-1}^{-1} (\boldsymbol{\theta}_t -  \boldsymbol{\hat{\theta}}_{t|t-1})  -\log \Big( \frac{1}{(2 \pi)^{d/2} |{\bf P}_{t|t-1}|^{1/2}}  \Big) \Big\} \\
	\nonumber & = \arg\min_{\boldsymbol{\theta}_i}  \Big\{ \frac{1}{2} \big( y({\bf u}_t) - h({\bf u}_t; \boldsymbol{\theta}_t) \big)^\top  {\bf P}_{{\bf n}_t}^{-1} \big( y({\bf u}_t) - h({\bf u}_t; \boldsymbol{\theta}_t) \big) + \frac{1}{2} (\boldsymbol{\theta}_t - \boldsymbol{\hat{\theta}}_{t|t-1})^\top  {\bf P}_{t|t-1}^{-1} (\boldsymbol{\theta}_t -  \boldsymbol{\hat{\theta}}_{t|t-1}) \Big\} 
	\end{align*}
	where $|\cdot|$ denotes the determinant. We receive the following objective function:
	\begin{align}
	\label{eq:Objective}
	L_t(\boldsymbol{\theta}_t) & =  \frac{1}{2} \big( y({\bf u}_t) - h({\bf u}_t; \boldsymbol{\theta}_t) \big)^\top  {\bf P}_{{\bf n}_t}^{-1} \big( y({\bf u}_t) - h({\bf u}_t; \boldsymbol{\theta}_t) \big) + \frac{1}{2} (\boldsymbol{\theta}_t - \boldsymbol{\hat{\theta}}_{t|t-1})^\top  {\bf P}_{t|t-1}^{-1} (\boldsymbol{\theta}_t -  \boldsymbol{\hat{\theta}}_{t|t-1})  
	\end{align}
	which is exactly the objective function (\ref{suppeq:objective_EKF}) in Theorem \ref{supp:theorem1}, with: $\delta ({\bf u}_t; \boldsymbol{\theta}_t) = y({\bf u}_t) - h({\bf u}_t; \boldsymbol{\theta}_t)$. To minimize this objective function we take the derivative of $L^{\text{EKF}}_t(\boldsymbol{\theta}_t)$ with respect to $\boldsymbol{\theta}_t$:
	\begin{align*}
	\nabla_{\boldsymbol{\theta}_t} L^{\text{EKF}}_t(\boldsymbol{\theta}_t) & = -     \nabla_{\boldsymbol{\theta}_t}  h({\bf u}_t,  \boldsymbol{\theta}_t) {\bf P}_{{\bf n}_t}^{-1} \big( y({\bf u}_t) - h({\bf u}_t; \boldsymbol{\theta}_t) \big)  + {\bf P}_{t|t-1}^{-1} (\boldsymbol{\theta}_t - \boldsymbol{\hat{\theta}}_{t|t-1})  = 0
	\end{align*}
	We use the linearization of the value function in Equation (\ref{suppeq:Linearization}):
	\begin{align*}
	{\bf P}_{t|t-1}^{-1} (\boldsymbol{\theta}_t - \boldsymbol{\hat{\theta}}) &  =  \nabla_{\boldsymbol{\theta}_t}  \big( h({\bf u}_t; \boldsymbol{\hat{\theta}}) +   \nabla_{\boldsymbol{\theta}_t} h({\bf u}_t; \boldsymbol{\hat{\theta}})^\top \big( \boldsymbol{\theta}_{t} - \boldsymbol{\hat{\theta}} \big)\big) {\bf P}_{{\bf n}_t}^{-1} \Big( y({\bf u}_t) -  h({\bf u}_t; \boldsymbol{\hat{\theta}}) -   \nabla_{\boldsymbol{\theta}_t} h({\bf u}_t; \boldsymbol{\hat{\theta}})^\top \big( \boldsymbol{\theta}_{t} - \boldsymbol{\hat{\theta}} \big)  \Big)    \\
	& =  \nabla_{\boldsymbol{\theta}_t} h({\bf u}_t; \boldsymbol{\hat{\theta}}) {\bf P}_{{\bf n}_t}^{-1} \Big( y({\bf u}_t) -  h({\bf u}_t; \boldsymbol{\hat{\theta}}) \Big) - \nabla_{\boldsymbol{\theta}_t} h({\bf u}_t; \boldsymbol{\hat{\theta}}) {\bf P}_{{\bf n}_t}^{-1} \nabla_{\boldsymbol{\theta}_t} h({\bf u}_t; \boldsymbol{\hat{\theta}})^\top  \big( \boldsymbol{\theta}_{t} - \boldsymbol{\hat{\theta}} \big)
	\end{align*}
	We receive that:
	\begin{align*}
	\Big( {\bf P}_{t|t-1}^{-1} + \nabla_{\boldsymbol{\theta}_t} h({\bf u}_t; \boldsymbol{\hat{\theta}}) {\bf P}_{{\bf n}_t}^{-1} \nabla_{\boldsymbol{\theta}_t} h({\bf u}_t; \boldsymbol{\hat{\theta}})^\top\Big)  (\boldsymbol{\theta}_t - \boldsymbol{\hat{\theta}}) =  \nabla_{\boldsymbol{\theta}_t} h({\bf u}_t; \boldsymbol{\hat{\theta}}) {\bf P}_{{\bf n}_t}^{-1} \Big( y({\bf u}_t) -  h({\bf u}_t; \boldsymbol{\hat{\theta}}) \Big)
	\end{align*}
	and finally:
	\begin{align}
	\label{suppeq:weight_MAP}
	\boldsymbol{\theta}_t & = \boldsymbol{\hat{\theta}} +  \Big( {\bf P}_{t|t-1}^{-1} +  \nabla_{\boldsymbol{\theta}_t} h({\bf u}_t; \boldsymbol{\hat{\theta}}) {\bf P}_{{\bf n}_t}^{-1} \nabla_{\boldsymbol{\theta}_t} h({\bf u}_t; \boldsymbol{\hat{\theta}})^\top\Big)^{-1}  \nabla_{\boldsymbol{\theta}_t} h({\bf u}_t; \boldsymbol{\hat{\theta}}) {\bf P}_{{\bf n}_t}^{-1} \Big( y({\bf u}_t) -  h({\bf u}_t; \boldsymbol{\hat{\theta}}) \Big)
	\end{align}
	For simplicity we denote:
	$\nabla {\bf h} = \nabla_{\boldsymbol{\theta}_t} h({\bf u}_t; \boldsymbol{\hat{\theta}})$. 
	We will now simplify the following term:
	\begin{align}
	\label{suppeq:Kt_derivation}
	& \nonumber \Big( {\bf P}_{t|t-1}^{-1} + \nabla {\bf h} {\bf P}_{{\bf n}_t}^{-1} \nabla {\bf h}^\top\Big)^{-1} \nabla {\bf h} {\bf P}_{{\bf n}_t}^{-1} \\
	\nonumber & = \Big( {\bf P}_{t|t-1}^{-1} + \nabla {\bf h} {\bf P}_{{\bf n}_t}^{-1} \nabla {\bf h}^\top\Big)^{-1} \nabla {\bf h} {\bf P}_{{\bf n}_t}^{-1} \Big( \nabla {\bf h}^\top {\bf P}_{t|t-1} \nabla {\bf h} + {\bf P}_{{\bf n}_t}\Big) \Big( \nabla {\bf h}^\top {\bf P}_{t|t-1} \nabla {\bf h} + {\bf P}_{{\bf n}_t}\Big)^{-1}\\
	\nonumber & = \Big( {\bf P}_{t|t-1}^{-1} + \nabla {\bf h} {\bf P}_{{\bf n}_t}^{-1} \nabla {\bf h}^\top\Big)^{-1}  \Big( \nabla {\bf h} {\bf P}_{{\bf n}_t}^{-1} \nabla {\bf h}^\top {\bf P}_{t|t-1} \nabla {\bf h}  + \nabla {\bf h} {\bf P}_{{\bf n}_t}^{-1} {\bf P}_{{\bf n}_t}\Big) \Big( \nabla {\bf h}^\top {\bf P}_{t|t-1} \nabla {\bf h} + {\bf P}_{{\bf n}_t}\Big)^{-1}\\
	\nonumber & = \Big( {\bf P}_{t|t-1}^{-1} + \nabla {\bf h} {\bf P}_{{\bf n}_t}^{-1} \nabla {\bf h}^\top\Big)^{-1}  \Big( \nabla {\bf h} {\bf P}_{{\bf n}_t}^{-1} \nabla {\bf h}^\top  +  {\bf P}_{t|t-1}^{-1} \Big) {\bf P}_{t|t-1} \nabla {\bf h} \Big( \nabla {\bf h}^\top {\bf P}_{t|t-1} \nabla {\bf h} + {\bf P}_{{\bf n}_t}\Big)^{-1} \\
	& = {\bf P}_{t|t-1} \nabla {\bf h} \Big( \nabla {\bf h}^\top {\bf P}_{t|t-1} \nabla {\bf h} + {\bf P}_{{\bf n}_t}\Big)^{-1}  \underbrace{=}_{(\ref{suppeq:covariance_weights_innovation})+(\ref{suppeq:covariance_innovation})} {\bf P}_{\boldsymbol{\tilde{\theta}}_t,{\bf \tilde{y}}_{t}}  {\bf P}_{{\bf \tilde{y}}_t}^{-1}  \underbrace{=}_{(\ref{suppeq:kalman_gain})} {\bf K}_t
	\end{align}
	Substituting this results in Equation (\ref{suppeq:weight_MAP}), we receive the EKF update for the parameters:
	\begin{align}
	\label{suppeq:EKF_weight}
	\boldsymbol{\hat{\theta}}^{\text{EKF}}_{t|t}  = \boldsymbol{\hat{\theta}}_{t|t-1} +  {\bf K}_t \big( y({\bf u}_t) - h({\bf u}_t; \boldsymbol{\hat{\theta}}_{t|t-1}) \big) 
	\end{align}
	which is exactly as in Equation (\ref{suppeq:kalman update}).
	
	We will now develop the term $\Big( {\bf P}_{t|t-1}^{-1} +\nabla {\bf h} {\bf P}_{{\bf n}_t}^{-1} \nabla {\bf h}^\top\Big)^{-1}$ that appears in (\ref{suppeq:weight_MAP}) by using the matrix inversion lemma:
	\begin{equation}
	\label{suppeq:MatrixInversionLemma}
	({\bf B}^{-1} + {\bf C}{\bf D}^{-1}{\bf C}^\top)^{-1} = {\bf B} - {\bf BC}({\bf D}+ {\bf C}^\top {\bf BC})^{-1} {\bf C}^\top{\bf B}
	\end{equation}
	where ${\bf B}$ is a square symmetric positive-definite (and hence invertible) matrix. For this purpose we assume that the error covariance matrix of $\boldsymbol{\theta}_t$, ${\bf P}_{t|t-1}$, is symmetric and positive-definite.
	\begin{align*}
	& \nonumber \Big( {\bf P}_{t|t-1}^{-1} + \nabla {\bf h} {\bf P}_{{\bf n}_t}^{-1} \nabla {\bf h}^\top\Big)^{-1}\\
	& \underbrace{=}_{(\ref{suppeq:MatrixInversionLemma})} {\bf P}_{t|t-1} - {\bf P}_{t|t-1} \nabla {\bf h}({\bf P}_{{\bf n}_t} + \nabla {\bf h}^\top {\bf P}_{t|t-1} \nabla {\bf h})^{-1}\nabla {\bf h}^\top {\bf P}_{t|t-1}  \underbrace{=}_{(\ref{suppeq:Kt_derivation})} {\bf P}_{t|t-1} - {\bf K}_t \nabla {\bf h}^\top {\bf P}_{t|t-1}\\
	\nonumber & \underbrace{=}_{(\ref{suppeq:covariance_weights_innovation})} {\bf P}_{t|t-1} - {\bf K}_t {\bf P}_{\boldsymbol{\tilde{\theta}}_t,{\bf \tilde{y}}_{t}}^\top \underbrace{=}_{(\ref{suppeq:kalman_gain})} {\bf P}_{t|t-1} - {\bf K}_t  {\bf P}_{{\bf \tilde{y}}_t} {\bf K}_t^\top
	\end{align*} 
	
	We can write the update of the parameters error covariance as:
	\begin{equation}
	\label{suppeq:error_covariance_update}
	\boxed { {\bf P}_{t|t} = {\bf P}_{t|t-1} - {\bf K}_t  {\bf P}_{{\bf \tilde{y}}_t} {\bf K}_t^\top }
	\end{equation}
	
	We conclude the proof by stating that the optimal parameter $\boldsymbol{\hat{\theta}}_{t|t}^{\text{EKF}}$ in (\ref{suppeq:kalman update}) is the solution to the minimization of the objective function in (\ref{suppeq:objective_EKF}):
	\[\boldsymbol{\hat{\theta}}_{t|t}^{\text{EKF}} \in \arg\min_{\boldsymbol{\theta}_{t}} L_t^{\text{EKF}}(\boldsymbol{\theta}_t)\] 
\end{proof}

\subsection{Proof of Colloraly 1}
\begin{proof}
	If ${\bf P}_{{\bf n}_t}$ is diagonal with diagonal elements $\sigma_i=N$, where $N$ is the number of samples in a batch, then:
	\begin{align*}
	\frac{1}{2}  \delta({\bf u}_t; \boldsymbol{\theta}_t)^\top {\bf P}_{{\bf n}_t}^{-1}  \delta({\bf u}_t; \boldsymbol{\theta}_t) & = \frac{1}{2 N} \sum_{i=1}^N  \delta^2(u_t^i,\boldsymbol{\theta}_t) = L^{\text{MLE}}_t (\boldsymbol{\theta}_t)
	\end{align*}
	
	If in addition, ${\bf P}_{0|0} = {\bf 0}$, and ${\bf P}_{{\bf v}_t} = {\bf 0}$ then the the initial error covariance matrix does not change and $L^{\text{EKF}}_t (\boldsymbol{\theta}_t) = L^{\text{MLE}}_t (\boldsymbol{\theta}_t)$ for each $t$.
\end{proof}

\subsection{Theorem 2}
The following Theorem formalizes the connection between $L_t^{\text{EKF}}$ and two separate KL-divergences ($D_{\text{KL}} (P || Q) = \int_{-\infty}^{\infty} p(x) \log (p(x)/q(x)) dx $): 
\begin{theorem}
	\label{supp:theorem2}
	Assume the inputs $u$ are drawn independently from a training distribution $\hat{Q}_{u}$, and the observations $y$ are drawn from a conditional training distribution $\hat{Q}_{y|u}$. Let $P_{u,y}(\boldsymbol{\theta})$ and $P_{y|u}(\boldsymbol{\theta})$ be the learned joint and conditional distributions, respectively. Define $\small C = \log \big(\frac{1}{(2\pi)^{N/2} |{\bf P}_{{\bf n}_t}|^{1/2} } \big)$. Under Assumptions \ref{suppAs:ConditionalIndependance} and \ref{suppAs:GaussianPosterior}, consider ${\bf P}_{{\bf n}_t}$ with diagonal elements $\sigma_i = N$, then:
	\begin{align*}
	L^{\text{EKF}}_t(\boldsymbol{\theta}_t)   = C +  N \mathbb{E}_{\hat{Q}_u} [D_{\text{KL}} \big(\hat{Q}_{y|u} || P_{y|u}(\boldsymbol{\theta}) \big)] +  t \cdot D_{\text{KL}} \big(P_{u,y}(\boldsymbol{\theta} + \Delta \boldsymbol{\theta})|| P_{u,y}(\boldsymbol{\theta}) \big)  + \mathcal{O}(\| \Delta \boldsymbol{\theta}\|^3)
	\end{align*}
\end{theorem}
Theorem \ref{supp:theorem2} illustrates how EKF minimizes two separate KL-divergences. The first is the KL divergence between two conditional distributions and it is equivalent to the loss in $L_t^{\text{MLE}}$. The second is the KL divergence between two different parameterizations of the joint learned distribution $P_{u,y}$. This term imposes {\em trust-region} on the VBF parameters in $L_t^{\text{EKF}}$, similarly to trust-region methods in policy optimization \cite{schulman2015trust}. 

To prove Theorem \ref{supp:theorem2} let's first define the distributions of interest. We adopt the notation from \cite{martens2014new}. Assume the inputs $u$ are drawn independently from a target distribution $Q_{u}$ with density function $q(u)$, and assume the corresponding outputs $y$ are drawn from a conditional target distribution $Q_{y|u}$ with density function $q(y|u)$. The target joint distribution is $Q_{u,y}$ whose density is $q(u,y) = q(y|u)q(u)$, and the learned distribution is $P_{u,y}(\boldsymbol{\theta})$, whose density is $p(u,y|\boldsymbol{\theta}) = p(y|u, \boldsymbol{\theta})q(u)$. 

\begin{lemma}
	\label{supp:Lemma1}
	If ${\bf P}_{{\bf n}_t}$ is diagonal with diagonal elements $\sigma_i=N$, then:
	\[\frac{1}{2}  \delta({\bf u}_t; \boldsymbol{\theta}_t) ^\top  {\bf P}_{{\bf n}_t}^{-1} \delta({\bf u}_t; \boldsymbol{\theta}_t) =  C +  N \mathbb{E}_{\hat{Q}_u} [D_{\text{KL}}(\hat{Q}_{y|u} || P_{y|u}(\boldsymbol{\theta}))]\]
\end{lemma}
\begin{proof}
	By definition:
	\[D_{\text{KL}}(Q_{u,y} || P_{u,y}(\boldsymbol{\theta})) = \int q(u, y) \log \frac{q(u, y)}{p(u,y|\boldsymbol{\theta})} dudy\]
	This is equivalent to the expected KL divergence over the conditional distributions:
	\[\mathbb{E}_{Q_u} [D_{\text{KL}}(Q_{y|u} || P_{y|u}(\boldsymbol{\theta}))] \]
	since:
	\begin{align*}
	& \mathbb{E}_{Q_u} [D_{\text{KL}}(Q_{y|u} || P_{y|u}(\boldsymbol{\theta}))] = \int q(u) \int q(y|u) \log \frac{q(y|u)}{p(y|u,\boldsymbol{\theta})} dydu\\
	& = \int q(u, y) \log \frac{q(y|u)q(u)}{p(y|u,\boldsymbol{\theta})q(u)} dudy = D_{\text{KL}}(Q_{u,y} || P_{u,y}(\boldsymbol{\theta}))
	\end{align*}
	Since we don't have access to $Q_u$ we substitute an empirical training distribution $\hat{Q}_u$ for $Q_u$ which is given by a set $\mathcal{S}_u$ of samples from $Q_u$. Then we define:
	\begin{align*}
	\mathbb{E}_{\hat{Q}_u} [D_{\text{KL}}(Q_{y|u} || P_{y|u}(\boldsymbol{\theta}))] & = \frac{1}{|\mathcal{S}|} \sum_{u \in \mathcal{S}_u} D_{\text{KL}}(Q_{y|u} || P_{y|u}(\boldsymbol{\theta}))
	\end{align*}
	In our training setting, we only have access to a single sample $y$ from $Q_{y|u}$ for each $u \in \mathcal{S}_u$, giving an empirical training distribution $\hat{Q}_{y|u}$. Then:
	\begin{align*}
	\mathbb{E}_{\hat{Q}_u} [D_{\text{KL}}(\hat{Q}_{y|u} || P_{y|u}(\boldsymbol{\theta}))] & = \frac{1}{|\mathcal{S}|} \sum_{(u, y) \in \mathcal{S}} 1 \log \frac{1}{p(y|u,\boldsymbol{\theta})} = - \frac{1}{|\mathcal{S}|} \sum_{(u, y) \in \mathcal{S}}  \log p(y|u,\boldsymbol{\theta})
	\end{align*}
	since $\hat{q}(y|u) = 1$.
	Now, back to our EKF notations. Assume that the $N$ observations in $y({\bf u}_t)$ are independent, then:
	\[\log p(y({\bf u}_t)|\boldsymbol{\theta}) = \log \Big(\prod_{i=1}^{N} p(y(u_t^i) | \boldsymbol{\theta}) \Big) = \sum_{i=1}^{N} \log p(y|u_t^i, \boldsymbol{\theta})\]
	where we changed the notation: $p(y(u_t^i) | \boldsymbol{\theta}) = p(y|u_t^i, \boldsymbol{\theta})$. Now let's write it explicitly for Gaussian distributions: 
	\begin{align*}
	\log p(y({\bf u}_t)|\boldsymbol{\theta}) &= \log \bigg( \frac{1}{(2\pi)^{N/2} |{\bf P}_{{\bf n}_t}|^{1/2} } \exp \Big( - \frac{1}{2} \big( y({\bf u}_t) - h({\bf u}_t; \boldsymbol{\theta}_t) \big)^\top  {\bf P}_{{\bf n}_t}^{-1} \big( y({\bf u}_t) - h({\bf u}_t; \boldsymbol{\theta}_t) \big) \Big) \bigg)\\
	& = C - \frac{1}{2} \big( y({\bf u}_t) - h({\bf u}_t; \boldsymbol{\theta}_t) \big)^\top  {\bf P}_{{\bf n}_t}^{-1} \big( y({\bf u}_t) - h({\bf u}_t; \boldsymbol{\theta}_t) \big)\\
	& = C - \frac{1}{2}  \delta({\bf u}_t; \boldsymbol{\theta}_t) ^\top  {\bf P}_{{\bf n}_t}^{-1} \delta({\bf u}_t; \boldsymbol{\theta}_t)
	\end{align*}
	where $C = \log \big(\frac{1}{(2\pi)^{N/2} |{\bf P}_{{\bf n}_t}|^{1/2} } \big)$ is constant with respect to $\boldsymbol{\theta}$.
	Then we have that:
	\begin{align*}
	& \frac{1}{2}  \delta({\bf u}_t; \boldsymbol{\theta}_t) ^\top  {\bf P}_{{\bf n}_t}^{-1} \delta({\bf u}_t; \boldsymbol{\theta}_t) = C - \log p(y({\bf u}_t)|\boldsymbol{\theta})\\
	& = C -  \sum_{i=1}^{N} \log p(y|u_t^i, \boldsymbol{\theta}) = C +  N \mathbb{E}_{\hat{Q}_u} [D_{\text{KL}}(\hat{Q}_{y|u} || P_{y|u}(\boldsymbol{\theta}))]
	\end{align*}
	
	In conclusion:
	\[ \boxed{\frac{1}{2}  \delta({\bf u}_t; \boldsymbol{\theta}_t) ^\top  {\bf P}_{{\bf n}_t}^{-1} \delta({\bf u}_t; \boldsymbol{\theta}_t) =  C +  N \mathbb{E}_{\hat{Q}_u} [D_{\text{KL}}(\hat{Q}_{y|u} || P_{y|u}(\boldsymbol{\theta}))]}\]
\end{proof}

\begin{lemma}
	\label{supp:Lemma2}
	For the empirical Fisher information matrix ${\bf \hat{F}}$:
	\begin{align*}
	D_{\text{KL}} \big(P_{u,y}(\boldsymbol{\theta} + \Delta \boldsymbol{\theta})|| P_{u,y}(\boldsymbol{\theta}) \big) & = \frac{1}{2}  (\boldsymbol{\theta} - \boldsymbol{\hat{\theta}})^T {\bf \hat{F}} (\boldsymbol{\theta} - \boldsymbol{\hat{\theta}}) + \mathcal{O}(\|\Delta \boldsymbol{\theta}\|^3)
	\end{align*}
\end{lemma}

\begin{proof}
	According to the KL-divergence definition:
	\begin{align*}
	& D_{\text{KL}} \big(P_{u,y}(\boldsymbol{\theta} + \Delta \boldsymbol{\theta})|| P_{u,y}(\boldsymbol{\theta}) \big)  =  \int p(u,y| \boldsymbol{\theta} + \Delta \boldsymbol{\theta}) \log p(u, y|\boldsymbol{\theta} + \Delta \boldsymbol{\theta}) dudy  - \int p(u, y| \boldsymbol{\theta} + \Delta \boldsymbol{\theta}) \log p(u,y| \boldsymbol{\theta}) dudy.
	\end{align*}
	
	According to Taylor expansion:
	\begin{align*}
	\nonumber \log p(u,y|\boldsymbol{\theta}) &= \log p(u,y | \boldsymbol{\theta} + \Delta \boldsymbol{\theta} ) - {\bf g}^T \Delta \boldsymbol{\theta} + \frac{1}{2} \Delta \boldsymbol{\theta}^T {\bf H} \Delta \boldsymbol{\theta} + \mathcal{O}(\|\Delta \boldsymbol{\theta}\|^3)
	\end{align*} 
	where ${\bf g}$ is the gradient of $\log p(u,y|\boldsymbol{\theta})$ at the point $\boldsymbol{\theta} + \Delta \boldsymbol{\theta}$:
	\[{\bf g} = \nabla_{\boldsymbol{\theta}} \log p(u,y| \boldsymbol{\theta})_{|\boldsymbol{\theta} + \Delta \boldsymbol{\theta}}.\]
	Note that $p(u,y| \boldsymbol{\theta})=p(y| u, \boldsymbol{\theta} + \Delta \boldsymbol{\theta}) q(u)$. Since $q(u)$ does not depend on $\boldsymbol{\theta}$ then $\nabla_{\boldsymbol{\theta}} \log p(u,y| \boldsymbol{\theta}) = \nabla_{\boldsymbol{\theta}} \log p(y|u, \boldsymbol{\theta})$. Therefore, we can write ${\bf g}$ as:
	\begin{equation*}
	{\bf g} = \nabla_{\boldsymbol{\theta}} \log p(y| u, \boldsymbol{\theta})_{|\boldsymbol{\theta} + \Delta \boldsymbol{\theta}} = \begin{bmatrix}
	\frac{\partial \log p(y|u, \boldsymbol{\theta} + \Delta \boldsymbol{\theta})}{\partial \theta_1} \\ \vdots \\ \frac{\partial \log  p(y|u, \boldsymbol{\theta} + \Delta \boldsymbol{\theta})}{\partial \theta_d}
	\end{bmatrix}
	\end{equation*}
	Similarly, the Hessian ${\bf H}$ can be written as:
	\begin{align*}
	{\bf H} & = \nabla^2_{\boldsymbol{\theta}} \log p(u,y|\boldsymbol{\theta})_{|\boldsymbol{\theta} + \Delta \boldsymbol{\theta}} = \nabla^2_{\boldsymbol{\theta}} \log p(y|u, \boldsymbol{\theta})_{|\boldsymbol{\theta} + \Delta \boldsymbol{\theta}} = \begin{bmatrix}
	\frac{\partial^2 \log p(y|u, \boldsymbol{\theta}+ \Delta \boldsymbol{\theta})}{\partial\theta_1^2} & \ldots & \frac{\partial^2 \log p(y|u, \boldsymbol{\theta} + \Delta \boldsymbol{\theta})}{\partial\theta_1 \partial\theta_d} \\ \vdots & \vdots & \vdots \\ \frac{\partial^2 \log p(y|u, \boldsymbol{\theta} + \Delta \boldsymbol{\theta})}{\partial\theta_d \partial\theta_1} & \ldots & \frac{\partial^2 \log p(y|u, \boldsymbol{\theta} + \Delta \boldsymbol{\theta})}{\partial \theta_d^2}
	\end{bmatrix}
	\end{align*}
	We use this Taylor expansion in the KL-divergence term, and use the notation: $\boldsymbol{\hat{\theta}} = \boldsymbol{\theta} + \Delta \boldsymbol{\theta} \quad \rightarrow \boldsymbol{\theta} -  \boldsymbol{\hat{\theta}} = - \Delta \boldsymbol{\theta}$.
	{\small 
		\begin{align*}
		& D_{\text{KL}} \big(P_{u,y}(\boldsymbol{\theta} + \Delta \boldsymbol{\theta})|| P_{u,y}(\boldsymbol{\theta}) \big) 
		=  \int p(u,y| \boldsymbol{\hat{\theta}}) \log p(u,y| \boldsymbol{\hat{\theta}}) dudy - \int p(u,y| \boldsymbol{\hat{\theta}}) \Big( \log p(u,y|  \boldsymbol{\hat{\theta}} ) - {\bf g}^T \Delta \boldsymbol{\theta} + \frac{1}{2} \Delta \boldsymbol{\theta}^T {\bf H} \Delta \boldsymbol{\theta} \Big)  dudy  + \mathcal{O}(\|\Delta \boldsymbol{\theta}\|^3)  \\
		& =  \underbrace{\int p(u,y| \boldsymbol{\hat{\theta}}) \log p(u,y| \boldsymbol{\hat{\theta}}) dudy - \int p(u,y| \boldsymbol{\hat{\theta}}) \log p(u,y| \boldsymbol{\hat{\theta}} ) dudy }_{=0} \\
		& + \underbrace{\int p(u,y|\boldsymbol{\hat{\theta}}) \sum_{i=1}^{d} \frac{\partial \log p(y|u, \boldsymbol{\hat{\theta}})}{\partial \theta_i} \Delta \theta_i dudy}_{=0, see (*)}  \underbrace{- \frac{1}{2} \int  p(u,y| \boldsymbol{\hat{\theta}})  \sum_{i=1}^{d} \sum_{j=1}^{d} \Delta \theta_i \Delta \theta_j \frac{\partial^2 \log p(y|u,\boldsymbol{\hat{\theta}})}{\partial\theta_i \partial\theta_j}  dudy}_{=\frac{1}{2}  \Delta \boldsymbol{\theta}^T {\bf F} \Delta \boldsymbol{\theta}, see (**)} + \mathcal{O}(\|\Delta \boldsymbol{\theta}\|^3)\\
		& = \frac{1}{2}  \Delta \boldsymbol{\theta}^T {\bf F} \Delta \boldsymbol{\theta} + \mathcal{O}(\|\Delta \boldsymbol{\theta}\|^3)\\
		\end{align*}
	}
	We explain (*), according to regularities in the Leibniz integral rule (switching derivation and integral):
	{\small 
		\begin{align*}
		\int p(u,y| \boldsymbol{\hat{\theta}}) \sum_{i=1}^{d} \frac{\partial \log p(y|u,\boldsymbol{\hat{\theta}})}{\partial \theta_i} \Delta \theta_i dudy & = \int q(u) \int p(y|u, \boldsymbol{\hat{\theta}}) \sum_{i=1}^{d} \frac{1}{p(y|u, \boldsymbol{\hat{\theta}})} \frac{\partial p(y|u, \boldsymbol{\hat{\theta}})}{\partial \theta_i} \Delta \theta_i dydu \\
		& = \int q(u) \sum_{i=1}^{d} \Delta \theta_i   \underbrace{\frac{\partial}{\partial \theta_i} \underbrace{\int p(y|u,\boldsymbol{\hat{\theta}}) dy}_{=1}}_{=0} du = 0
		\end{align*}
	}
	We explain (**):
	{\small 
		\begin{align*}
		& - \frac{1}{2} \int  p(u,y| \boldsymbol{\hat{\theta}})  \sum_{i=1}^{d} \sum_{j=1}^{d} \Delta \theta_i \Delta \theta_j \frac{\partial^2 \log p(y|u, \boldsymbol{\hat{\theta}})}{\partial\theta_i \partial\theta_j} du dy \\
		& = - \frac{1}{2} \int q(u) \sum_{i=1}^{d} \sum_{j=1}^{d} \Delta \theta_i \Delta \theta_j \cdot \int  p(y|u, \boldsymbol{\hat{\theta}})   \frac{\partial }{\partial\theta_i} \Big(\frac{1}{ p(y|u, \boldsymbol{\hat{\theta}})}\frac{\partial p(y|u, \boldsymbol{\hat{\theta}})}{ \partial\theta_j} \Big)  dy du\\
		& = - \frac{1}{2} \int q(u) \sum_{i=1}^{d} \sum_{j=1}^{d} \Delta \theta_i \Delta \theta_j \int  p(y|u, \boldsymbol{\hat{\theta}})   \Big( \frac{1}{ p(y|u, \boldsymbol{\hat{\theta}})} \frac{\partial^2 p(y|u, \boldsymbol{\hat{\theta}})}{\partial\theta_i \partial\theta_j}  - \frac{1}{ p(y|u, \boldsymbol{\hat{\theta}})^2} \frac{\partial p(y|u, \boldsymbol{\hat{\theta}}) }{\partial\theta_i} \frac{\partial p(y|u, \boldsymbol{\hat{\theta}})}{ \partial\theta_j}   \Big) dy du\\
		& = - \frac{1}{2} \int q(u) \sum_{i=1}^{d} \sum_{j=1}^{d} \Delta \theta_i \Delta \theta_j \int  \Big( \frac{\partial^2 p(y|u, \boldsymbol{\hat{\theta}})}{\partial\theta_i \partial\theta_j}  - p(y|u, \boldsymbol{\hat{\theta}}) \frac{\partial \log p(y|u, \boldsymbol{\hat{\theta}}) }{\partial\theta_i} \frac{\partial \log p(y|u, \boldsymbol{\hat{\theta}})}{ \partial\theta_j} \Big)  dy du\\
		& = - \frac{1}{2} \int q(u) \sum_{i=1}^{d} \sum_{j=1}^{d} \Delta \theta_i \Delta \theta_j \underbrace{\frac{\partial^2}{\partial\theta_i \partial\theta_j} \underbrace{\int p(y|u, \boldsymbol{\hat{\theta}}) dy}_{=1}}_{=0} du  + \frac{1}{2} \int q(u) \sum_{i=1}^{d} \sum_{j=1}^{d} \Delta \theta_i \Delta \theta_j \mathbb{E}_{P_{y|u}(\boldsymbol{\hat{\theta}})} \Big[  \frac{\partial \log p(y|u, \boldsymbol{\hat{\theta}}) }{\partial\theta_i} \frac{\partial \log p(y|u, \boldsymbol{\hat{\theta}})}{ \partial\theta_j} \Big] du\\
		& = \frac{1}{2}  \Delta \boldsymbol{\theta}^T {\bf F} \Delta \boldsymbol{\theta}
		\end{align*}
	}
	where
	\[{\bf F}_{ij} = \mathbb{E}_{Q_{u}} \Bigg[  \mathbb{E}_{P_{y|u}(\boldsymbol{\hat{\theta}})} \Big[  \frac{\partial \log p(y|u, \boldsymbol{\hat{\theta}}) }{\partial\theta_i} \frac{\partial \log p(y|u, \boldsymbol{\hat{\theta}})}{ \partial\theta_j} \Big] \Bigg]\]
	Since we don't have access to $Q_u$ we will use the empirical training distribution $\hat{Q}_u$:
	\[{\bf \hat{F}}_{ij} = \frac{1}{|\mathcal{S}|} \sum_{u \in \mathcal{S}_u}   \mathbb{E}_{P_{y|u}(\boldsymbol{\hat{\theta}})} \Big[  \frac{\partial \log p(y|u, \boldsymbol{\hat{\theta}}) }{\partial\theta_i} \frac{\partial \log p(y|u, \boldsymbol{\hat{\theta}})}{ \partial\theta_j} \Big] \]
	We received that:
	\begin{align*}
	D_{\text{KL}} \big(P_{u,y}(\boldsymbol{\theta} + \Delta \boldsymbol{\theta})|| P_{u,y}(\boldsymbol{\theta}) \big) = 
	\frac{1}{2}  (\boldsymbol{\theta} - \boldsymbol{\hat{\theta}})^T {\bf \hat{F}} (\boldsymbol{\theta} - \boldsymbol{\hat{\theta}}) + \mathcal{O}(\|\Delta \boldsymbol{\theta}\|^3)
	\end{align*}
\end{proof}
Now we can summarize the proof for Theorem 2: 
\begin{proof}
	Adding the relationship from \cite{ollivier2018online}: ${\bf \hat{F}}_{t|t-1} = \frac{1}{t} {\bf P}_{t|t-1}^{-1}$, and combining the results from Lemma \ref{supp:Lemma1} and Lemma \ref{supp:Lemma2}, our objective function can be approximated as:
	{\small 
		\begin{align*}
		& L^{\text{EKF}}_t(\boldsymbol{\theta}_t)	 = \frac{1}{2}  \delta({\bf u}_t; \boldsymbol{\theta}_t)^\top {\bf P}_{{\bf n}_t}^{-1}  \delta({\bf u}_t; \boldsymbol{\theta}_t) +  \frac{1}{2}(\boldsymbol{\theta}_t - \boldsymbol{\hat{\theta}}_{t|t-1})^\top {\bf P}_{t|t-1}^{-1} (\boldsymbol{\theta}_t - \boldsymbol{\hat{\theta}}_{t|t-1}) \\
		& = C +  N \mathbb{E}_{\hat{Q}_u} [D_{\text{KL}}(\hat{Q}_{y|u} || P_{y|u}(\boldsymbol{\theta}))]  +  \frac{t}{2}(\boldsymbol{\theta}_t - \boldsymbol{\hat{\theta}}_{t|t-1})^\top {\bf \hat{F}}_{t|t-1} (\boldsymbol{\theta}_t - \boldsymbol{\hat{\theta}}_{t|t-1})\\
		& \approx C +  N \mathbb{E}_{\hat{Q}_u} [D_{\text{KL}}(\hat{Q}_{y|u} || P_{y|u}(\boldsymbol{\theta}))]  +  t \cdot D_{\text{KL}} \big(P_{u,y}(\boldsymbol{\theta} + \Delta \boldsymbol{\theta})|| P_{u,y}(\boldsymbol{\theta}) \big) \\
		\end{align*}
	}
	which completes the proof.
\end{proof}

\newpage 
\section{KTD and GPTD models}
For completeness, we present here the Kalman Temporal Difference (KTD) algorithm \cite{geist2010kalman} and the Gaussian Process Temporal Difference (GPTD) algorithm \cite{engel2003bayes,engel2005reinforcement}, which we discussed and compared to KOVA in the main paper.

KTD formulates the parameter estimation problem through the following model: 
\begin{equation*}
\label{eq:UKF}
\begin{dcases}
\boldsymbol{\theta}_t = \boldsymbol{\theta}_{t-1} + {\bf v}_t\\
r_t = \begin{dcases}
\hat{V}(s_t; \boldsymbol{\theta}_t) - \gamma \hat{V}(s_{t+1}; \boldsymbol{\theta}_t) + n_t \\
\hat{Q}(s_t, a_t; \boldsymbol{\theta}_t) - \gamma \hat{Q}(s_{t+1}, s_{t+1}; \boldsymbol{\theta}_t) + n_t \\
\hat{Q}(s_t, a_t; \boldsymbol{\theta}_t) - \gamma \max_{a} \hat{Q}(s_{t+1}, a; \boldsymbol{\theta}_t) + n_t
\end{dcases} 
\end{dcases}
\end{equation*} 
where $\boldsymbol{\theta}_t \in \mathbb{R}^{d \times 1}$ is a parameter vector evaluated at time $t$, ${\bf v}_t$ is additive white evolution noise with covariance ${\bm P}_{{\bm v}_t}$ and $n_t$ is additive white observation noise with variance $P_{n_t}$. $\hat{V}(\cdot; \boldsymbol{\theta})$ and $\hat{Q}(\cdot; \boldsymbol{\theta})$  are non-linear approximations of the state value function and state-action value function, respectively, with parameters $\boldsymbol{\theta}$. 
In KTD, both the rewards $r_t$ and the parameters $\boldsymbol{\theta}$ are modeled as RVs.

GPTD formulates the parameter estimation problem through the following model: 
\begin{equation*}
\label{eq:GPTD}
\begin{cases}
\boldsymbol{\theta}_t = \boldsymbol{\theta}_{t-1}\\
{\small \begin{pmatrix}
	r_1\\ \vdots \\ r_t
	\end{pmatrix} = \begin{pmatrix}
	1 & - \gamma & 0 & \cdots & \cdots\\ 0 & 1 & -\gamma & \cdots & 0\\ \vdots & \ddots & \ddots & \ddots & \vdots\\ 0 & \cdots & \cdots & 1 & -\gamma
	\end{pmatrix} \begin{bmatrix}
	\phi(s_1)^\top \\ \vdots \\ \phi(s_t)^\top \\ \phi(s_{t+1})^\top
	\end{bmatrix} \boldsymbol{\theta}_t + 
	\begin{pmatrix}
	n_1\\ \vdots \\ n_t
	\end{pmatrix}}
\end{cases}
\end{equation*}
where $\phi(s_t) \in \mathbb{R}^{d \times 1}$ are state features. In GPTD, the noise $n_j$ is assumed white, Gaussian and of variance $\sigma_j$, and the prior over parameters is assumed to follow a normal distribution. 

\subsection{Motivating example}
\begin{figure}[H]
	\centering{
		\includegraphics[width=.3\linewidth,keepaspectratio]{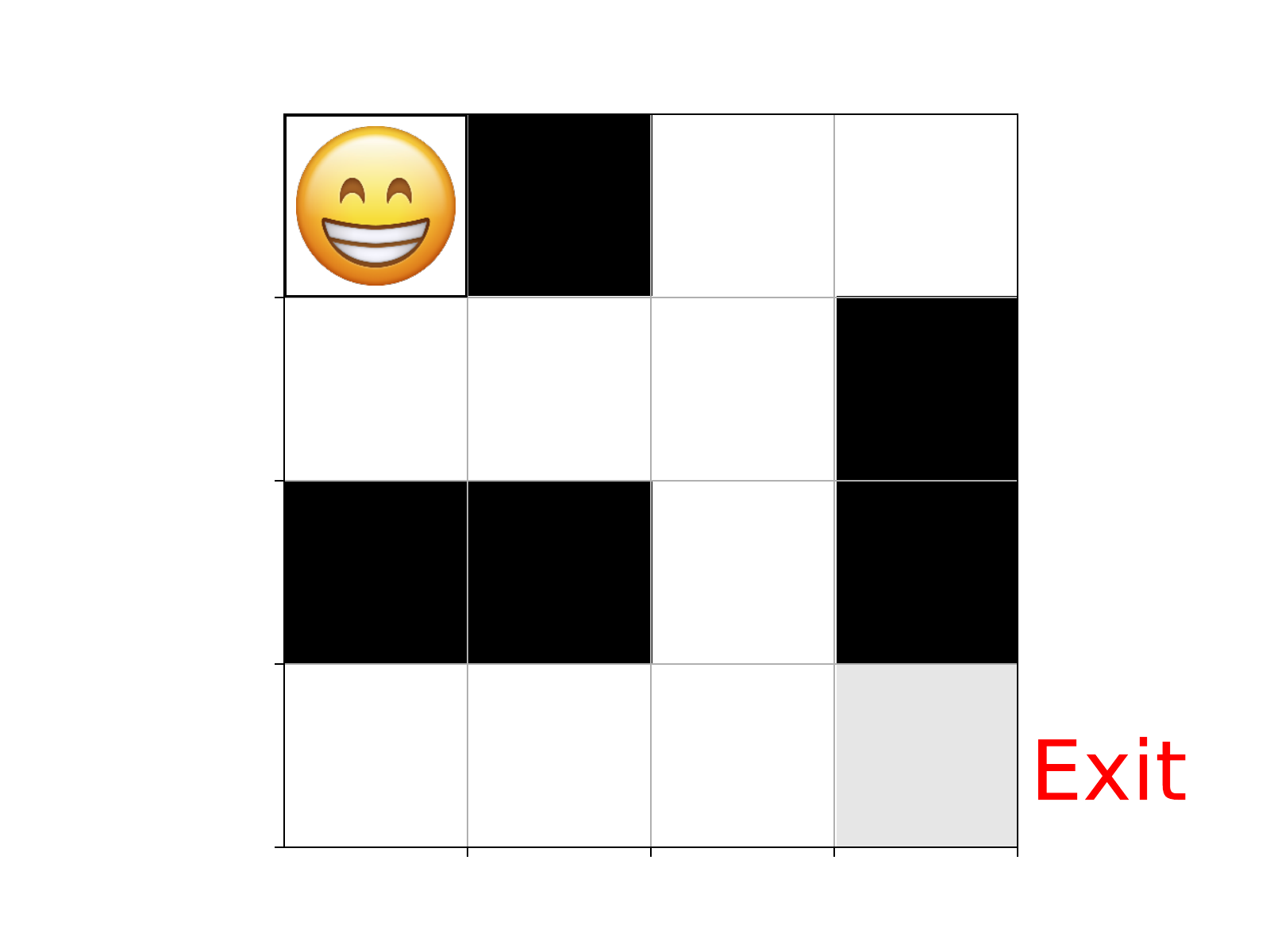}} 
	\caption{The environment we used in the motivating example is a ($4\!\times\!4$)-maze.} 
	\label{fig:maze_img_4x4}
\end{figure}
In our motivating example (Figure 1 in the main paper) we tested KOVA vs KTD on a $(4\times4)$-maze (Figure \ref{fig:maze_img_4x4}). A full description of this environment can be found in Section \ref{supp:MazeEnv}. The VBF (Q-function) is approximated with a fully connected MultiLayer Perceptron (MLP): The input size is $16$ ($4\times4$ image); one hidden layer of $16$ units and ReLU nonlinearity; output size is $4$ (number of actions). The total number of parameters in this network is $d=340$. The agent uses a discount factor of $\gamma=0.95$ and an $\epsilon$-greedy policy for selecting actions with $\epsilon=0.1$. For KTD we set batch-size $N=1$ (this is obligated by KTD), $P_{n_t}=1$, ${\bf P}_{0|0} = 10{\bf I}$ and $\eta_{\text{KTD}}=0.01$ where ${\bf P}_{{\bf v}_t} = \eta_{\text{KTD}}{\bf P}_{t-1|t-1}$. For KOVA we set batch-size $N=32$, ${\bf P}_{{\bf n}_t}=N{\bf I}$, ${\bf P}_{0|0} = {\bf I}$ and $\eta=0.01$ where ${\bf P}_{{\bf v}_t} = (\eta /\ 1-\eta) {\bf P}_{t-1|t-1}$. The sample generator $\mathcal{R}$ is an experience buffer that contains transitions form different policies. 

{\bf Running time}: We trained both KTD and KOVA for $5,000$ timesteps. The running time for KTD was $13,716$ seconds while for KOVA it was $577$ seconds. The main reason for this extreme difference in running time is due to the number of feed forward passes required by KTD in each optimization step, as explained in the main paper.

\newpage
\section{Visual Illustration of KOVA} 
We add here some visual illustrations for KOVA. In Figure \ref{fig:kalman_distribution_diagram} the Kalman perspective for policy evaluation is presented. In Figure \ref{fig:kalman_block_diagram} a block diagram for estimating the VBF parameters is presented.  
\begin{figure}[H]
	\centering{
		\includegraphics[width=.6\linewidth,keepaspectratio]{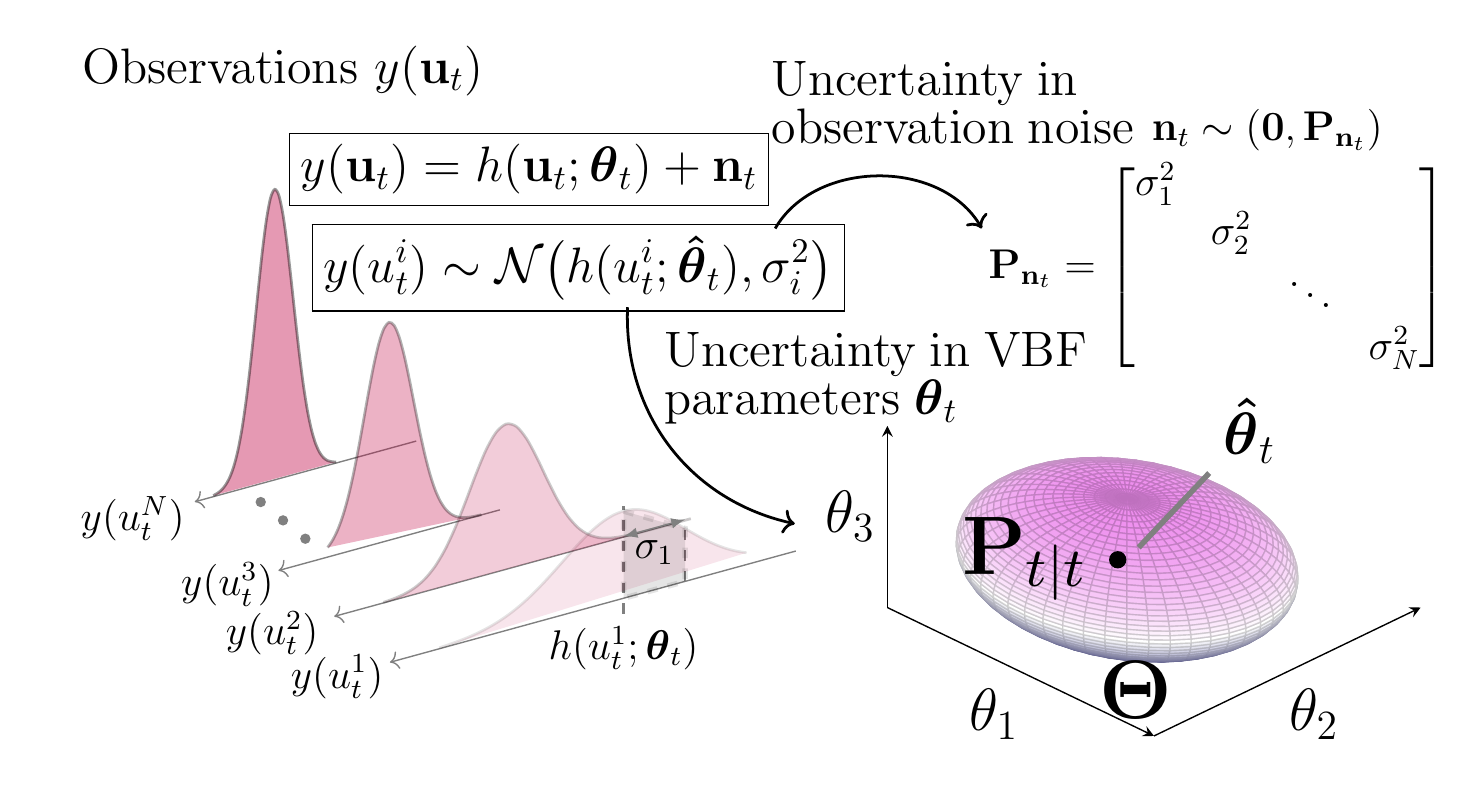}} 
	\caption{A Kalman perspective for the policy evaluation problem in RL. The randomness of a noisy observation $y(u_t)$ originates from two sources: (i) the randomness of its mean, the VBF $h(u_t; \boldsymbol{\theta}_t)$ through the dependency on the random parameters $\boldsymbol{\theta}_t$. (ii) the random zero-mean noise ${\bf n}_t$ which relates to stochastic transitions and to the possibly random policy.} 
	\label{fig:kalman_distribution_diagram}
\end{figure}

\begin{figure}[H]
	\centering{
		\includegraphics[width=0.7\linewidth,height=0.5\textheight,keepaspectratio]{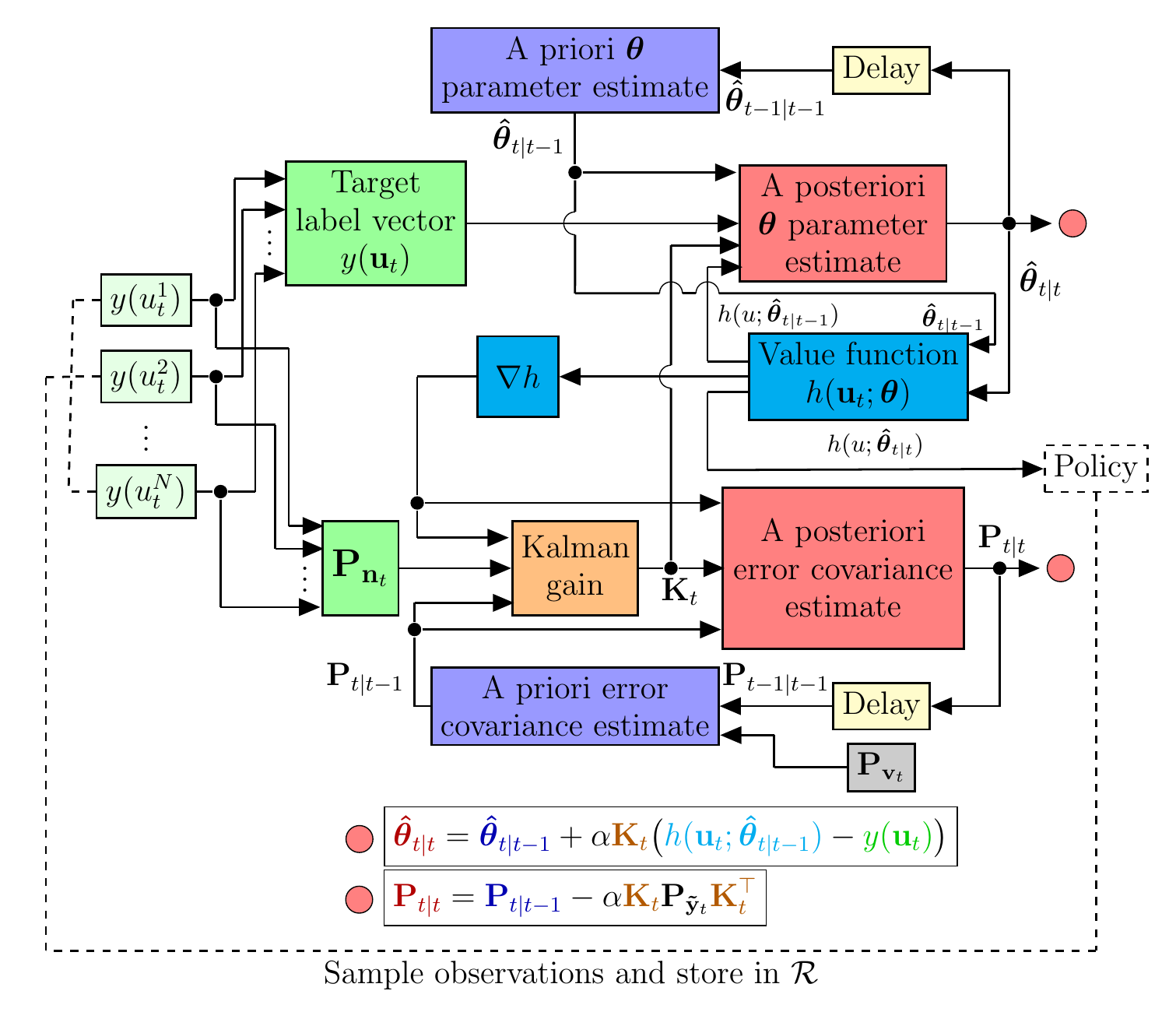}}
	\caption{KOVA optimizer block diagram. KOVA receives as input the initial general prior ${\bf P}_{0|0}$ and the covariances ${\bf P}_{{\bf v}_t}$ and ${\bf P}_{{\bf n}_t}$. It initializes $\boldsymbol{\hat{\theta}}_{0|0}$ with small random values or with the VBF parameters of the previous policy. For every $t$, it samples $N$ target labels from $\mathcal{R}$ (see Table 1 in main paper for target label examples), constructs $y({\bf u}_t)$ and $ h({\bf u}_t$, $\boldsymbol{\hat{\theta}}_{t|t-1})$ and computes $\nabla_{\boldsymbol{\theta}_t} h({\bf u}_t; \boldsymbol{\hat{\theta}}_{t|t-1})$ and the Kalman gain ${\bf K}_t$. Then it updates and outputs the MAP parameters estimator $\boldsymbol{\hat{\theta}}_{t|t}$ and the error covariance matrix ${\bf P}_{t|t}$.}
	\label{fig:kalman_block_diagram}
\end{figure}

\newpage
\section{Experimental details} 
\subsection{Hyper-parameters and grid search} 
\label{supp:hyper}
In our KOVA implementations we used a grid-search over the following hyper-parameters:
\begin{itemize}
	\item KOVA learning rate=\{1.0, 0.1, 0.01\}.
	\item ${\bf P}_{{\bf n}_t}$ type = \{batch-size, max-ratio\}. The {\it batch-size} setting is a diagonal matrix with $\sigma_i=\sigma=N$. The {\it max-ratio} setting is a diagonal matrix with $\sigma_i=N \max (1, \frac{1}{\frac{\pi_{\text{old}} (a_i|s_i)}{\pi_{\text{new}}(a_i|s_i)} + \epsilon_1})$ with $\epsilon_1=10^{-5}$.
	\item $\eta$=\{0.1, 0.01, 0.001\}, where ${\bf P}_{{\bf v}_t} = \frac{\eta}{1-\eta}{\bf P}_{t-1|t-1}$ with $\eta$ being a small number that controls the amount of fading memory.
\end{itemize} 

For Adam optimizer we tried different learning rates: Adam learning rate=$\{10^{-3}, 3 \cdot 10^{-4}, 10^{-4}\}$. For each learning rate value we tried both a constant setting and a decaying setting. 

The final hyper-parameters that we chose in our experiments appear in the tables below.

\subsection{Maze environment}
\label{supp:MazeEnv}
In the maze experiment we used a ($10\!\times\!10$)-maze as illustrated in Figure \ref{fig:maze_img}. \footnote{Our experiment is based on a Q-learning maze implementation from https://www.samyzaf.com/ML/rl/qmaze.html.}
\begin{figure}[H]
	\centering{
		\includegraphics[width=.6\linewidth,keepaspectratio]{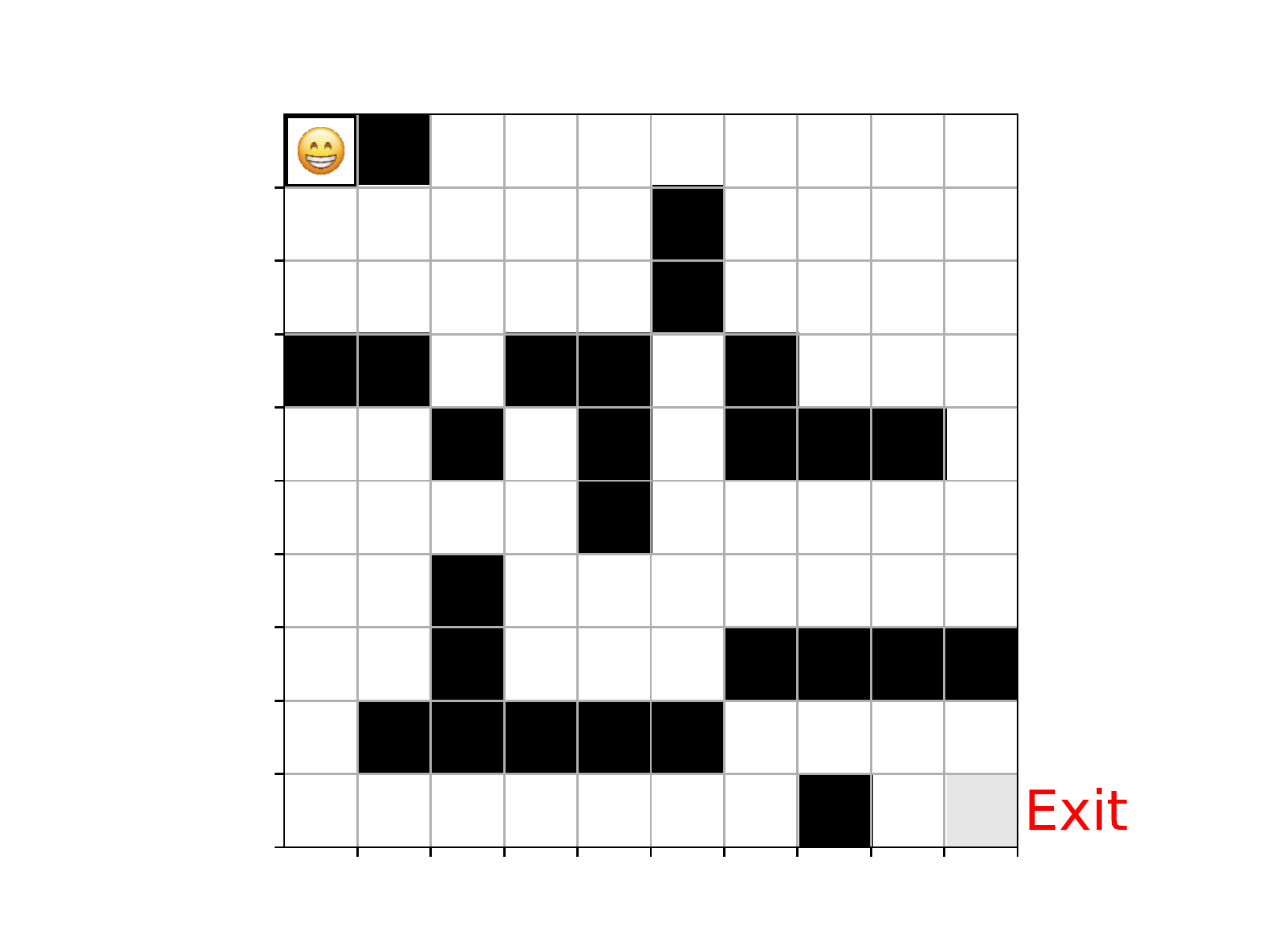}} 
	\caption{The environment we used for the maze experiment is a ($10\!\times\!10$)-maze. The agent is randomly placed at one of the free (white) cells of the maze. The bottom-right corner is the exit.} 
	\label{fig:maze_img}
\end{figure}

{\bf State space}: States are represented as images of the maze. In our $(10\!\times\!10)$-maze, $s_t \in \mathbb{R}^{10\times10}$.

{\bf Action space}: The action space is discrete with four possible actions: 'top', 'down', 'right' or 'left.

{\bf Rewards}: The agent receives $-0.04$ reward for arriving to a new cell; it receives $-0.25$ reward for returning to pre-visited cells; it receives a reward of $1$ when arriving to the exit at the bottom-right corner of the maze.  

{\bf Experiment setting} An episode begins when an agent is randomly placed at one of the maze free cells. The agent travels in the maze where its goal is to arrive to the exit. A success is defined if the agent arrives to the exit at the bottom-right corner. A loss is defined when the total reward is below $-50$ and the agent haven't arrived yet to the exit. In both cases the episode ends. The VBF (Q-function) is approximated with a fully connected MLP: The input size is $100$ ($10\times10$ image);  one hidden layer of $100$ units and ReLU nonlinearity; output size is $4$ (number of actions) . The total number of parameters in this network is $d=10,504$. The agent uses an $\epsilon$-greedy policy for selecting actions with $\epsilon=0.1$. For KOVA, ${\bf P}_{0|0} = {\bf I}$ and the sample generator $\mathcal{R}$ is an experience buffer that contains transitions form different policies.

We used double Q-learning \cite{van2016deep}, an off-policy algorithm, and tested KOVA vs. Adam for learning the Q-function parameters.  In Tables \ref{Qlearning_table} and \ref{Kalman_for_Qleaning_table} we present the hyper parameters for this experiments.

\begin{table*}[htp]
	\parbox[t]{.55\linewidth}{
		\centering
		\caption{Double Q-learning hyper-parameters used for maze.}
		\begin{tabular}{|p{42mm}|p{47mm}|} 
			\hline
			{\bf Hyper-parameter} & {\bf Value} \\ \hline \hline
			Minibatch size & $32$ \\ \hline
			Adam learning rate & few values were tested: \newline $\{10^{-3}, 3 \cdot 10^{-4}\text{-decaying}, 10^{-4}\}$  \\ \hline
			Discount $(\gamma)$ & $0.95$ \\ \hline
			Total timesteps & $100,000$ \\ \hline
			$t_\text{update}$ (Target network  \newline update every $t_\text{update}$ timesteps) & $200$ \\ \hline
		\end{tabular}
		\label{Qlearning_table}
	} \quad \quad
	\parbox[t]{.35\linewidth}{
		\centering
		\caption{KOVA hyper-parameters used for Q-function optimization in maze.}
		\begin{tabular}{|l|p{35mm}|}
			\hline
			{\bf Hyper-parameter} & {\bf Value} \\ \hline \hline
			KOVA learning rate & $1.0$ \\ \hline
			${\bf P}_{{\bf n}_t}$ type & batch-size \\ \hline
			$\eta$ &  few values were tested: \newline $\{0.1, 0.01, 0.001\}$  \\ \hline
		\end{tabular}
		\label{Kalman_for_Qleaning_table}
	}
\end{table*} 

{\bf Running time}: In Table \ref{clock_time_maze} we compare between the running time of the algorithms we tested on the maze environment (presented in Figure 2 in the main paper). We can see that the running time of KOVA is longer by 29\% compared to Adam optimizer.

\begin{table*}[htp]
	\centering
	\caption{Running time (seconds) for maze experiment. We trained over $100,000$ timesteps.}
	\begin{tabular}{|c|c|}
		\hline
		& {\bf ${\bf (10\!\times\!10)}$-maze} \\ \hline \hline
		{\bf Adam} & 7710 \\ \hline
		{\bf KOVA} & 9910 \\ \hline
		{\bf KOVA increased time percentage (\%)} & $29\%$ \\ \hline
	\end{tabular}
	\label{clock_time_maze}
\end{table*}

\subsection{Mujoco environment}

Our experiments are based on the baselines implementation  \cite{baselines} for PPO, TRPO and ACKTR, and on the stable-baseline implementation \cite{stable-baselines} for SAC. We used their default hyper parameters, and only changed the optimizer for the value function from Adam to KOVA. For brevity, we bring here the network architecture and the hyper parameters for each algorithm. 

{\bf PPO:}
Following \citep{schulman2017proximal}, the policy network is a fully-connected MLP with two hidden layers, 64 units and tanh nonlinearities. The output of the policy network is the mean and standard deviations of a Gaussian distribution of actions for a given (input) state.
The value network is a fully-connected MLP with two hidden layers, 64 units and tanh nonlinearities. The output of the value network is a scalar, representing the value function for a given (input) state. The policy loss for PPO is $\mathbb{\hat{E}}_{\phi} \big[ \min(r_t(\phi) \hat{A}_t, \text{clip} (r_t(\phi)) , 1-\epsilon_p, 1+\epsilon_p) \hat{A}_t \big] $ where $r_t(\phi) = \frac{\pi_\phi (a_t|s_t)}{\pi_{\phi_\text{old}} (a_t|s_t)}$, $\phi$ are the policy parameters and $\hat{A}_t$ is the GAE estimator for the advantage function \cite{schulman2015high}. The policy entropy is $\mathbb{\hat{E}}_t [\log(\sigma_{\pi_{\phi}} \sqrt{2 \pi e })]$. For KOVA, we set the initial prior ${\bf P}_{0|0} = {\bf I}$ and the sample generator $\mathcal{R}$ contains transitions from the current policy. In Tables \ref{PPO_table} and \ref{Kalman_for_PPO_table} we present the hyper parameters for the PPO experiments. The Horizon represents the number of timesteps per each policy rollout. 

{\bf TRPO:}
The policy network and the value network are the same as described for PPO, only with 32 units instead of 64. The policy loss for TRPO is $\mathbb{\hat{E}}_{\phi} \big[ \frac{\pi_\phi (a_t|s_t)}{\pi_{\phi_\text{old}} (a_t|s_t)}\big] \hat{A}_t$, where $\phi$ are the policy parameters and $\hat{A}_t$ is the GAE estimator. The policy entropy is $\mathbb{\hat{E}}_t [\log(\sigma_{\pi_{\phi}} \sqrt{2 \pi e })]$. For KOVA, we set the initial prior ${\bf P}_{0|0} = {\bf I}$ and the sample generator $\mathcal{R}$ contains transitions from the current policy. In Tables \ref{TRPO_table} and \ref{Kalman_for_TRPO_table} we present the hyper parameters for the TRPO experiments.

{\bf SAC:}
Following \citep{haarnoja2018soft}, the policy network is a fully-connected MLP with two hidden layers, 64 units and ReLU nonlinearities. The output of the policy network is the mean and standard deviations of a Gaussian distribution of actions for a given (input) state.
The value network and the Q-function networks are fully-connected MLP with two hidden layers, 64 units and ReLU nonlinearities. The output of the value network is a scalar, representing the value function for a given (input) state. The output of the Q-function network is a scalar, representing the value function for a given state and action. For KOVA, we set the initial prior ${\bf P}_{0|0} = {\bf I}$ and the sample generator $\mathcal{R}$ is an experience buffer that contains transitions form different policies.  In Tables \ref{SAC_table} and \ref{Kalman_for_SAC_table} we present the hyper parameters for the SAC experiments.

{\bf ACKTR:}
Following \cite{wu2017scalable}, there are two separate neural networks with 64 hidden units per layer in a two-layer network. The Tanh and ReLU nonlinearities are used for the policy network and value network, respectively, for all layers except the output layer, which doesn't have any nonlinearity. The log standard deviation of a Gaussian policy is parameterized as a bias in the final layer of policy network that doesn't depend on input state. ACKTR uses the asynchronous advantage actor critic (A3C) method for the advantage function estimation \cite{mnih2016asynchronous}. We used the baselines ACKTR implementation \cite{baselines} with all their default hyper-parameters, expect of the horizon (which in this algorithm is also the batch-size). In \cite{wu2017scalable}, ACKTR was trained with a horizon of 2500. At our experiments we tried horizon=\{250,  2500, 25000\} and chose to use horizon=250 in our experiments which gained the best results. In Table \ref{ACKTR_table} we present the hyper parameters we used for the ACKTR experiments.

\begin{table*}[htp]
	\parbox[t]{.45\linewidth}{
		\centering
		\caption{PPO hyper-parameters used for the Mujoco tasks}
		\begin{tabular}{|c|c|}
			\hline
			{\bf Hyper-parameter} & {\bf Value} \\ \hline \hline
			Horizon &  2048 \\ \hline 
			Minibatch size & 64 \\ \hline
			Adam learning rate & $3 \cdot 10^{-4}$ (decaying)  \\ \hline
			Num. epochs & 10 \\ \hline
			Discount $(\gamma)$ & 0.99 \\ \hline
			GAE parameter $(\lambda)$ & 0.95 \\ \hline
			Clip range $\epsilon_p$& 0.2 \\ \hline
		\end{tabular}
		\label{PPO_table}
	} \quad \quad
	\parbox[t]{.50\linewidth}{
		\centering
		\caption{KOVA hyper-parameters used for value function optimization in PPO}
		\begin{tabular}{|l|l|}
			\hline
			{\bf Hyper-parameter} & {\bf Value} \\ \hline \hline
			KOVA learning rate & $1.0$ (Swimmer, HalfCheetah,   \\ 
			& \quad \ \ \ Walker2d, Ant) \\ 
			& $0.1$ (Hopper) \\ \hline
			${\bf P}_{{\bf n}_t}$ type & max-ratio \\ \hline
			$\eta$ &  $0.1$ (Hopper, HalfCheetah) \\ 
			& $0.01$ (Swimmer, Walker2d)  \\ 
			& $0.001$ (Ant)  \\ \hline
		\end{tabular}
		\label{Kalman_for_PPO_table}
	}
\end{table*} 

\begin{table*}[htp]
	\parbox[t]{.45\linewidth}{
		\centering
		\caption{TRPO hyper-parameters used for Mujoco tasks}
		\begin{tabular}{|c|c|}
			\hline
			{\bf Hyper-parameter} & {\bf Value} \\ \hline \hline
			Horizon &  1024 \\ \hline 
			Batch size & 64 \\ \hline
			Discount $(\gamma)$ & 0.99 \\ \hline
			GAE parameter $(\lambda)$ & 0.98 \\ \hline
			Max KL & 0.01 \\ \hline
			Conjugate gradient iterations & 10 \\ \hline
			Conjugate gradient damping & 0.1 \\ \hline
			Value function iterations & 5 \\ \hline
			Value function learning rate & $10^{-3}$ \\ \hline
			Normalize observations & True \\ \hline
		\end{tabular}
		\label{TRPO_table}
	}
	\quad \quad 
	\parbox[t]{.50\linewidth}{
		\centering
		\caption{KOVA hyper-parameters used for value function optimization in TRPO}
		\begin{tabular}{|l|l|}
			\hline
			{\bf Hyper-parameter} & {\bf Value} \\ \hline \hline
			KOVA learning rate & $1.0$ (Swimmer, Hopper)  \\ 
			& $0.1$ (HalfCheetah) \\ 
			& $0.01$ (Ant, Walker2d) \\ \hline
			${\bf P}_{{\bf n}_t}$ type & max-ratio \\ \hline
			$\eta$ & $0.01$ (Swimmer, Walker2d)\\ 
			& $0.001$ (Hopper, HalfCheetah, Ant) \\ \hline
		\end{tabular}
		\label{Kalman_for_TRPO_table}
	}
\end{table*}


\begin{table*}[htp]
	\parbox[t]{.45\linewidth}{
		\centering
		\caption{SAC hyper-parameters used for Mujoco tasks}
		\begin{tabular}{|c|c|}
			\hline
			{\bf Hyper-parameter} & {\bf Value} \\ \hline \hline
			Batch size & 64 \\ \hline
			Adam learning rate & $3 \cdot 10^{-4}$ \\ \hline
			Discount $(\gamma)$ & 0.99 \\ \hline
			Replay buffer size & $50,000$ \\ \hline
			Target smoothing coefficient $(\tau)$ & $0.005$ \\ \hline
		\end{tabular}
		\label{SAC_table}
	}
	\quad \quad 
	\parbox[t]{.50\linewidth}{
		\centering
		\caption{KOVA hyper-parameters used for value function optimization in SAC}
		\begin{tabular}{|l|l|}
			\hline
			{\bf Hyper-parameter} & {\bf Value} \\ \hline \hline
			KOVA learning rate & $1.0$ (Swimmer, HalfCheetah)  \\ 
			& $0.1$ (Hopper, Ant, Walker2d) \\ \hline 
			${\bf P}_{{\bf n}_t}$ type & batch-size \\ \hline
			$\eta$ & $0.1$ (Ant) \\ 
			& $0.01$ (HalfCheetah, Walker2d) \\ 
			& $0.001$ (Swimmer, Hopper) \\ \hline
		\end{tabular}
		\label{Kalman_for_SAC_table}
	}
\end{table*}

\begin{table}[htp]
	\centering
	\caption{ACKTR hyper-parameters used for Mujoco tasks}
	\begin{tabular}{|c|c|}
		\hline
		{\bf Hyper-parameter} & {\bf Value} \\ \hline \hline 
		Horizon (is also the batch size)&  250 \\ \hline
		Discount $(\gamma)$ & 0.99 \\ \hline
		Entropy coefficient & 0.01 \\ \hline
		Value function coefficient & 0.5 \\ \hline
		Value function Fisher coefficient & 1.0 \\ \hline
		Learning rate & 0.25 \\ \hline
		Max gradient norm & 0.5 \\ \hline
		KFAC clip & 0.001 \\ \hline
		Momentum & 0.9 \\ \hline
		Damping $\epsilon_d$ & 0.01 \\ \hline
	\end{tabular}
	\label{ACKTR_table}
\end{table}

\subsection{Algorithms running time}

In Table \ref{clock_time} we compare between the running time of the algorithms we tested on the Mujoco environment (presented in Figure 3 in the main paper). We can see that for PPO and TRPO, the running time of KOVA is longer by 44-60\% compared to Adam optimizer. For SAC, the running time for KOVA is much longer than Adam. Further research is needed to understand the reason and to develop solutions that would decrease the running time. ACKTR running time is similar to TRPO, although in our experiments it gained the lowest mean episode reward over time. 

\begin{table*}[htp]
	\centering
	\caption{Running time (seconds) for each tested algorithm in each tested environment. We trained over 1 million timesteps for Swimmer, Hopper and HalfCheetah. We trained over 2 million timesteps for Ant, and Walker2d.}
	\begin{tabular}{|c|c|c|c|c|c|}
		\hline
		& {\bf Swimmer-v2} & {\bf Hopper-v2} & {\bf HalfCheetah-v2} & {\bf Ant-v2} & {\bf Walker2d-v2} \\ \hline \hline
		{\bf PPO+ADAM} & 1820 & 1900 & 1880 & 6260 & 3950\\ \hline
		{\bf PPO+KOVA} & 2690 & 2790 & 2770 & 9850 & 5700\\ \hline
		{\bf PPO: KOVA increased time percentage (\%)} & $48\%$ & $47\%$ & $47\%$ & $57\%$ & $44\%$\\ \hline \hline
		{\bf TRPO+ADAM} & 1010 & 1070 & 1030 & 2720 & 2270  \\ \hline
		{\bf TRPO+KOVA} & 1570 & 1720 & 1620 & 4210 & 3540 \\ \hline
		{\bf TRPO: KOVA increased time percentage (\%)} & $55\%$ & $60\%$ & $57\%$ & $55\%$ & $56\%$\\ \hline \hline
		{\bf SAC+ADAM} & 12300 & 21110 & 13255 & 26250 & 24010 \\ \hline
		{\bf SAC+KOVA} & 43530 & 116070 & 48980 & 152940 & 108610\\ \hline 
		{\bf SAC: KOVA increased time percentage (\%)} & $254\%$ & $450\%$ & $269\%$ & $482\%$ & $352\%$\\ \hline \hline
		{\bf ACKTR} & 1120 & 1220 & 1140 & 3220 & 2630\\ \hline
	\end{tabular}
	\label{clock_time}
\end{table*}

\subsection{Additional experiment}

\subsubsection{Investigating the evolution and observation noises:}

As discussed in Section 3.4 in the main paper, the most interesting hyper-parameters in KOVA are related to the covariances ${\bf P}_{{\bf v}_t}$ and ${\bf P}_{{\bf n}_t}$. If we define the evolution noise covariance as ${\bf P}_{{\bf v}_t} = \frac{\eta}{1-\eta}{\bf P}_{t-1|t-1}$ then we can control the hyper-parmeter $\eta$. We investigated the effect of different values of $\eta$ and ${\bf P}_{{\bf n}_t}$ in the Swimmer and HalfCheetah environments, where KOVA gained the most success. The results are depicted in Figure \ref{supfig:entropy_ppo}. We tested two different ${\bf P}_{{\bf n}_t}$ settings: the {\it batch-size} setting and the {\it max-ratio} setting as explained in Section \ref{supp:hyper}. Interestingly, although using KOVA results in lower policy loss (which we try to maximize), it actually increases the policy entropy and encourages exploration, which we believe helps in gaining higher rewards during training. We can clearly see how the mean rewards increases as the policy entropy increases, for different values of $\eta$. This insight was observed in both tested Mujoco environments and in both settings of ${\bf P}_{{\bf n}_t}$.

\begin{figure}[H] 
	\centering
	\subfigure[]{\label{fig:entropy_ppo_max_ratio}\includegraphics[width=0.9\linewidth,height=0.6\textheight,keepaspectratio]{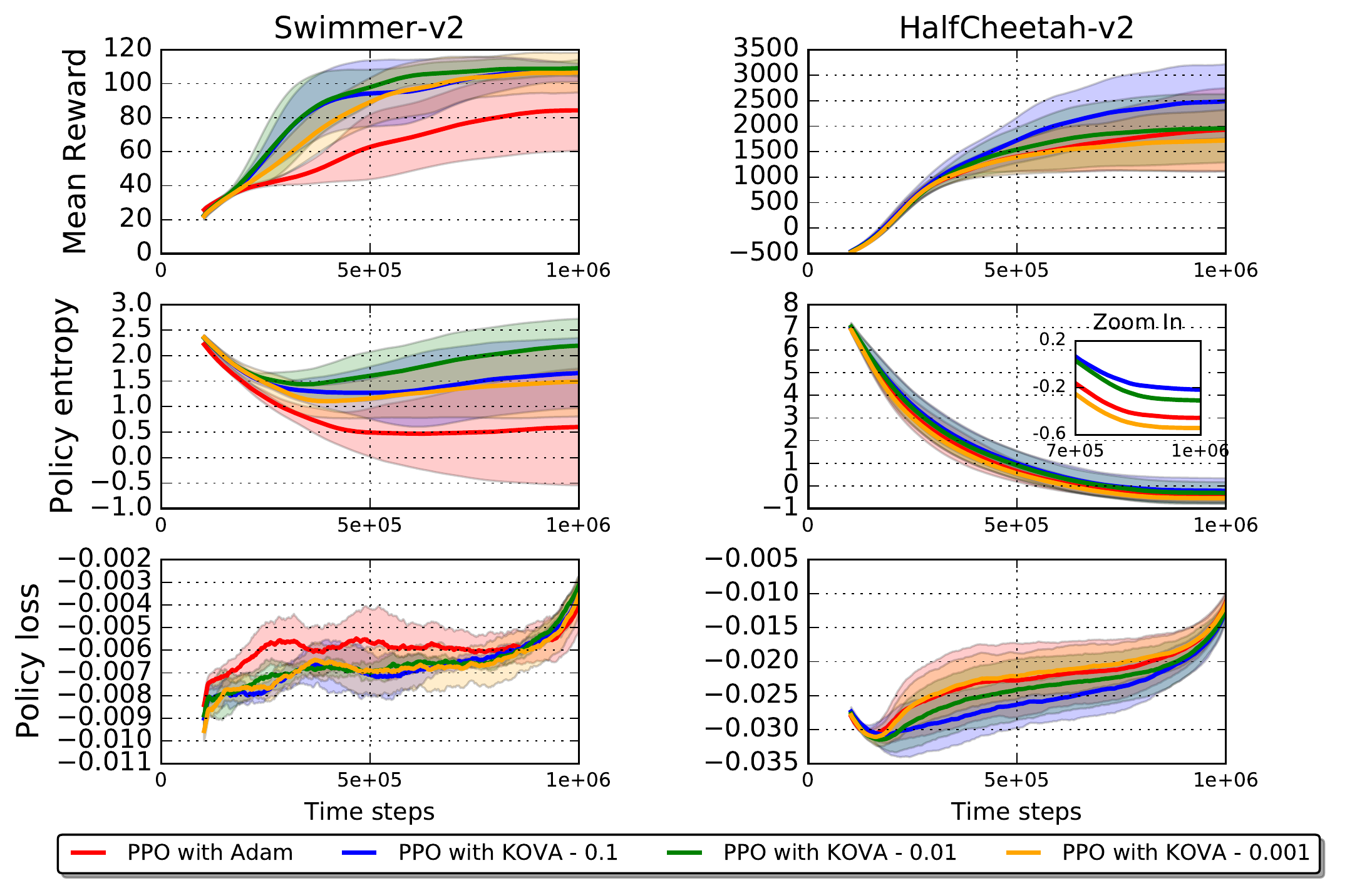}}
	\subfigure[]{\label{fig:entropy_ppo_batch_size}\includegraphics[width=0.9\linewidth,height=0.6\textheight,keepaspectratio]{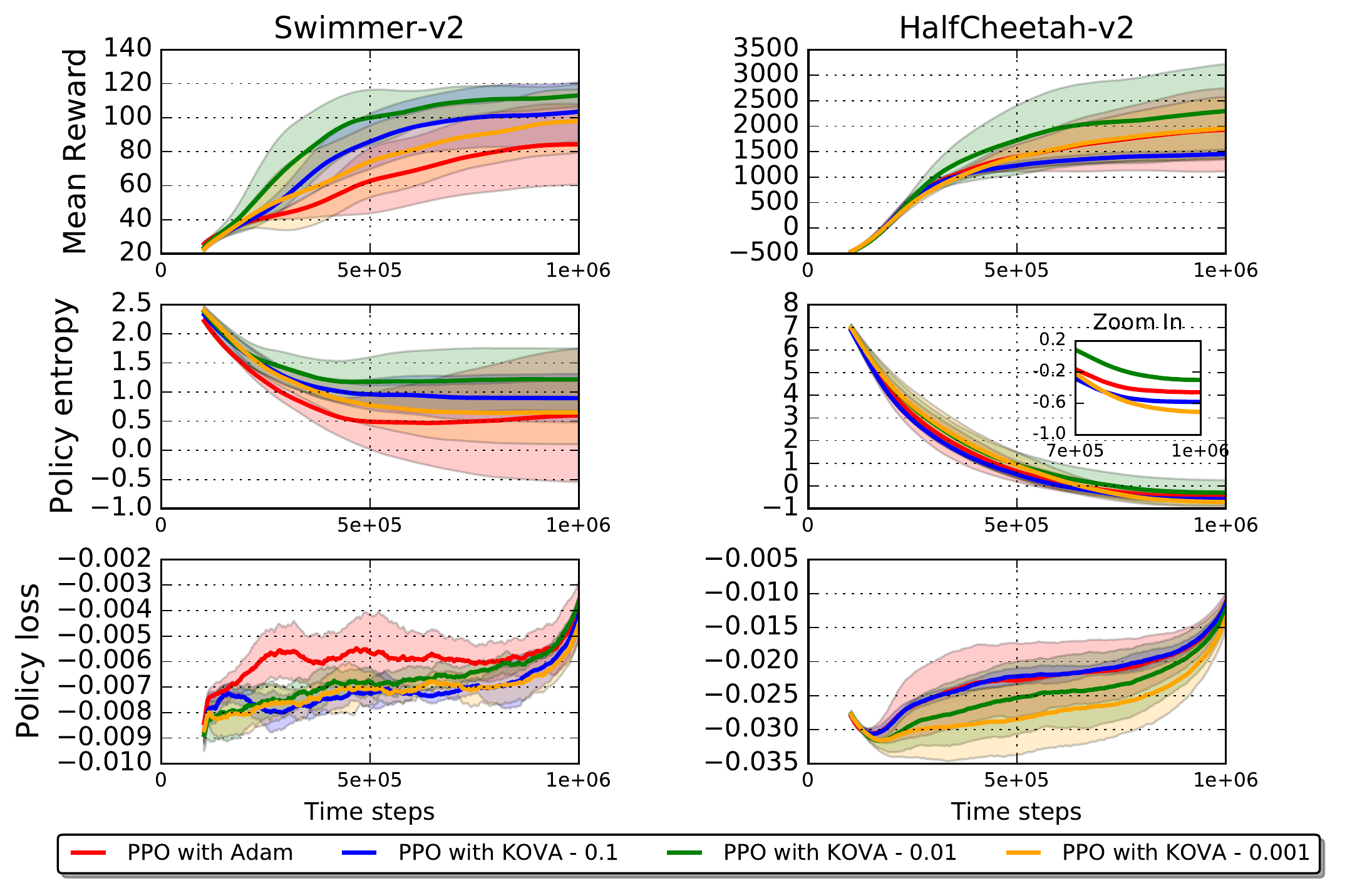}}
	\caption{Mean episode reward, policy entropy and the policy loss for a PPO agent in Swimmer-v2 and HalfCheetah-v2. We compare between optimizing the value function  with Adam vs. our KOVA optimizer. We present three different values for $\eta = 0.1, 0.01, 0.001$ and two different values for the diagonal elements in ${\bf P}_{{\bf n}_t}$: {\bf (a)} max-ratio and {\bf (b)} batch-size. We present the average (solid lines) and standard deviation (shaded area) of the mean episode reward over 8 runnings, generated from random seeds.} 
	\label{supfig:entropy_ppo}
\end{figure}

\small{
\bibliography{KalmanOptimizationforValueApproximation}
\bibliographystyle{icml2020}}

\end{document}